\documentclass{article}

\usepackage[preprint]{neurips_2026}

\usepackage[utf8]{inputenc} 
\usepackage[T1]{fontenc}    
\usepackage[hypertexnames=false]{hyperref}       
\usepackage{url}            
\usepackage{booktabs}       
\usepackage{amsfonts}       
\usepackage{nicefrac}       
\usepackage{microtype}      
\usepackage{xcolor}         
\usepackage{amsmath}
\usepackage{amsthm}
\usepackage{algorithm}
\usepackage{algpseudocode}
\usepackage{graphicx}

\newtheorem{theorem}{Theorem}
\newtheorem*{supptheorem}{Theorem}
\newtheorem{definition}{Definition}

\newtheorem{lemma}{Lemma}

\title{An Interpretable and Scalable Framework for Evaluating Large Language Models}

\author{%
  \textbf{Xinhao Qu}$^1$ \quad
  \textbf{Qiang Heng}$^{2}$ \quad
  \textbf{Hao Zeng}$^3$ \quad
  \textbf{Xiaoqian Liu}$^{1}$\thanks{Corresponding author. Email: \href{xiaoqian.liu@ucr.edu}{xiaoqian.liu@ucr.edu}} \\[2.5mm]
  $^1$Department of Statistics, University of California, Riverside\\
  $^2$School of Statistics and Data Science, Southeast University \\
  $^3$Department of Statistics and Data Science,
    Southern University of Science and Technology
}

\begin{document}

\maketitle

\begin{abstract}

Evaluation of large language models (LLMs) is increasingly critical, yet standard benchmarking methods rely on average accuracy,  overlooking both the inherent stochasticity of LLM outputs and the heterogeneity of benchmark items.  Item Response Theory (IRT) offers a principled framework for modeling latent model abilities and item characteristics, but conventional methods are computationally expensive and numerically unstable, limiting large-scale implementations.  To address these challenges, we propose an interpretable and scalable framework for LLM evaluation based on the majorization–minimization principle. Our approach reformulates the problem as a sequence of constrained matrix factorization subproblems, enabling stable and efficient parameter estimation with theoretical guarantees for identifiability and convergence.
Experiments on synthetic and real-world datasets, including MATH-500 and six Open LLM Leaderboard benchmarks, demonstrate that our method achieves superior scalability and interpretability. It delivers orders-of-magnitude speedups over competing methods while maintaining comparable or even higher estimation accuracy. Our results align with established scaling laws and offer insights into item difficulty and discrimination, informing more principled benchmark design. 

\end{abstract}

\section{Introduction}\label{sec-intro}
Given the rapid evolution of large language model (LLM) systems and their expanding applications \citep{intro1,intro2,intro3}, rigorous evaluation of LLMs is indispensable. Benchmarks have been the gold standard for developers to demonstrate model capabilities \citep{intro4,intro5}, with performance typically measured by average accuracy between ground-truth labels and model-generated outputs across all benchmark items. However, this conventional approach overlooks two fundamental challenges: the stochastic nature of LLM outputs and the heterogeneity of benchmark items.
First, LLM outputs can vary substantially under minor perturbations, such as changes in prompt order, slight typos, or the use of synonyms \citep{intro6,intro7,intro8}. Under greedy decoding settings, outputs can be  inconsistent even when presented with identical prompts \citep{intro9}, leading to fragile evaluations and complicating subsequent interpretation.
Second, benchmark items exhibit substantial heterogeneity in their evaluative utility, in contrast to the implicit equal-weight assumption
underlying average accuracy. Highly discriminative items can better differentiate model capabilities and thus warrant greater weights \citep{intro10}, whereas items compromised by training data contamination lose their evaluative integrity and should be downweighted \citep{intro11}. Furthermore, as benchmarks evolve, some items become obsolete or redundant \citep{intro12}, which reduces evaluation efficiency and warrants exclusion.

These challenges call for evaluation frameworks that can recover intrinsic model abilities under stochastic generation while capturing heterogeneity across benchmark items. Recently, psychometricians have brought Item Response Theory (IRT) into the picture \citep{irt3,intro17,lart}, framing LLM evaluation as a latent variable modeling problem analogous to human performance assessment. IRT models each binary evaluation outcome as a sample drawn from a latent distribution, parameterized by both model traits and item features. Estimating these parameters enables a systematic decomposition of model performance, moving beyond average accuracy toward a more interpretable and diagnostically informative evaluation paradigm.

However, traditional IRT-based evaluation methods \citep{mirt,lart} that rely on the expectation-maximization (EM) algorithm \citep{EM} suffer from  critical limitations, most notably \textbf{computational inefficiency} and \textbf{numerical instability}. These methods are often prohibitively slow \citep{virt,pyirt}, and their performance is extremely sensitive to singular data configurations, such as near-constant responses and collinearity \citep{irtNA,irt2}. Such issues are particularly pronounced in LLM evaluation settings, where training data leakage or extremely difficult items frequently arise \citep{intro11,redun2}. As a result, existing IRT approaches exhibit \textbf{poor scalability} and remain largely restricted to relatively \textbf{small-scale applications}, making them ill-suited for the rapidly growing ecosystem of LLMs and benchmark datasets.



To address these limitations, we propose an interpretable and scalable computational framework for LLM evaluation based on the majorization-minimization (MM) principle \citep{MM3,MM4}. Specifically, we reformulate the original problem as a sequence of constrained matrix factorization subproblems,
thereby stabilizing and accelerating the estimation procedure. We further address parameter identifiability and establish theoretical guarantees for algorithmic convergence. Our primary contributions are  as follows:
\begin{itemize}
  \item \textbf{A novel MM-based framework}: We develop an interpretable and scalable computational framework for LLM evaluation grounded in MM optimization, well-suited to large-scale evaluation datasets and capable of offering reliable and meaningful interpretations.

 \item \textbf{Theoretical guarantees}: We provide theoretical guarantees for parameter identifiability and algorithmic convergence, ensuring reliable recovery of model abilities and item features.

 \item \textbf{Computational efficiency \& scalability}: Extensive large-scale experiments demonstrate orders-of-magnitude speedup over competing methods while achieving comparable or even superior estimation accuracy.

 \item \textbf{Real-world applications \& empirical insights}: We evaluate on MATH-500 and six benchmarks from the Hugging Face Open LLM Leaderboard. Results on model abilities are consistent with established parametric scaling laws \citep{scalinglaw,scalinglaw2}, and the recovered item features align with human annotations \citep{MATH500}, capturing fine-grained heterogeneity that average accuracy fails to reflect.
\end{itemize}
Overall, this work presents a novel MM-based computational framework with rigorous theoretical guarantees, enhanced computational efficiency and scalability, and strong empirical interpretability, representing a substantial advancement over conventional IRT-based approaches for LLM evaluation.





\section{Preliminary}
We first conceptualize the evaluation process within a latent probabilistic framework.
Given a set of $N$ LLMs and $J$ benchmark items, let $y_{ij} \in \{-1, 1\}$ denote whether the $i$-th LLM correctly addressed the $j$-th benchmark item, yielding a binary response matrix $\mathbf{Y} \in \{-1, 1\}^{N \times J}$.
We assume that an LLM's output is governed by an intrinsic stochastic process, modeling the probability of a correct response via a logit link function:
\begin{equation}\label{equ:logit}
\Phi(x_{ij}) \equiv \mathbb{P}\left(y_{ij}=1 \mid x_{ij}\right) =\frac{1}{1+\exp{(-x_{ij}/\sigma)}},
\end{equation}
where $x_{ij}$ represents the pre-activation score produced by the $i$-th model for the  response of item $j$, and $\sigma$ acts as the scaling parameter. This formulation mirrors standard next-token decoding, where the multi-class softmax over the vocabulary collapses into a logistic distribution when projected onto the binary correctness space, with $\sigma$ corresponding directly to the temperature \citep{intro16}.

To obtain interpretable representations for fine-grained evaluation, we adopt a reparametrization based on the two-parameter IRT model (2PL-IRT) \citep{irt1,irt2}, expressing the score matrix as
\begin{equation}\label{equ:repara}
\mathbf{X} = \boldsymbol\theta \mathbf{a}^\top + \mathbf{1} \mathbf{b}^\top = \begin{pmatrix}
        \boldsymbol\theta & \mathbf{1}
    \end{pmatrix}\begin{pmatrix}
        \mathbf{a}^\top\\
        \mathbf{b}^\top
    \end{pmatrix} \equiv \mathbf{U}\mathbf{V}^\top,
\end{equation}
where $\boldsymbol\theta = (\theta_1, \cdots, \theta_N)^\top$, $\mathbf{a}=(a_1, \cdots, a_J)^\top$, $\mathbf{b}=(b_1, \cdots, b_J)^\top$, and $\mathbf{1} \in \mathbb{R}^N$ is a vector of ones.
Each entry $x_{ij}$ is decomposed as $x_{ij} = a_j \theta_i + b_j$, where $a_j$ denotes item discrimination, $\theta_i$ represents model ability, and  $b_j$ is an item easiness parameter. This formulation reveals a critical coupling: items with larger $a_j$ amplify differences in model ability, yielding stronger discriminative power, whereas items with small $a_j$ provide limited signal for distinguishing model capabilities.
Following standard 2PL-IRT frameworks, we impose $a_j$  to be non-negative to ensure a monotonic relationship between model ability and response accuracy, preserving interpretability of the discrimination parameter.
The intercept $b_j$ captures item easiness, serving as a baseline for the score that is independent of specific model-item interactions.

\section{Methodology}
To recover the triplet representation $({\boldsymbol\theta},{\mathbf{a}},{\mathbf{b}})$, a straightforward approach is to minimize a binary cross-entropy loss accounting for all evaluation entries
\begin{equation}\label{equ:loss}
 \ell (\mathbf{X}) =  - \sum_{(i,j)\in\Omega} \{\gamma_{ij} \log \Phi(x_{ij}) + (1-\gamma_{ij})\log(1-\Phi(x_{ij}))\},
\end{equation}
with $\mathbf{X} = \boldsymbol\theta \mathbf{a}^\top + \mathbf{1} \mathbf{b}^\top$, subject to $\mathbf{a} \geq \mathbf{0}$. Here, the scaled response variable is defined as $\gamma_{ij} = \frac{1}{2}(1+{y}_{ij})$, and $\Omega \subseteq \{(i, j) \mid 1 \le i \le N, 1 \le j \le J\}$ denotes the set of observed entries. For $(i,j)\notin\Omega$, we set $y_{ij}=0$.

\subsection{The majorization-minimization principle}
We employ the MM principle \citep{MM3,MM4} to convert the challenging minimization of $\ell (\mathbf{X})$ into a sequence of simpler surrogate problems. This is achieved through  a majorization approximation $g(\mathbf{X}\mid \tilde{\mathbf{X}})$ anchored at the current estimate $\tilde{\mathbf{X}}$, satisfying (i) a tangency condition $g(\tilde{\mathbf{X}}\mid \tilde{\mathbf{X}}) = \ell(\tilde{\mathbf{X}})$ for all $\tilde{\mathbf{X}}$ and (ii) a domination condition $g(\mathbf{X}\mid \tilde{\mathbf{X}}) \geq \ell(\mathbf{X})$ for all $\mathbf{X}$. The two conditions together ensure a monotone decrease of the objective $ \ell (\mathbf{X})$ through iterating
\begin{equation}\label{equ:iter}
\mathbf{X}^{(t+1)} = \arg\min_{\mathbf{X}} g(\mathbf{X}\mid \mathbf{X}^{(t)})
\end{equation}
for iteration step $t=0,1,\cdots$.
Notably, exact minimization of the surrogate function is not required to ensure descent of the objective. Therefore, one can inexactly solve (\ref{equ:iter}) and still guarantee a monotonic decrease.

Given the log-likelihood $\log \Phi(x)$ is $L$-Lipschitz differentiable with $L=1/4\sigma^2$, and that the density function $\phi (x) = \Phi'(x)$ is symmetric around zero, we invoke Proposition 2.1 in \cite{MMGN} to derive a quadratic majorization of $\ell (\mathbf{X})$ at $\tilde{\mathbf{X}}$, expressed as
\begin{equation}\label{equ:quad-maj}
g(\mathbf{X}\mid \tilde{\mathbf{X}}) = \frac{L}{2}\left\|\mathbf{X} - \tilde{\mathbf{Y}}\right\|_{\text{F}(\Omega)}^2 + c(\tilde{\mathbf{X}}),
\end{equation}
where $\tilde{\mathbf{Y}}=\tilde{\mathbf{X}} + 4\sigma{\mathbf{Y}}\circ\Phi(-{\mathbf{Y}}\circ\tilde{\mathbf{X}})$ under the  logit link (\ref{equ:logit}) with $\circ$ denoting the Hadamard product, and $c(\tilde{\mathbf{X}})$ is a constant depending on ${\tilde{\mathbf{X}}}$. Here $\left\| \cdot \right\|_{\text{F}(\Omega)}$ denotes the observed-entry Frobenius seminorm, defined for a matrix $\mathbf{A}$ as $\left\| \mathbf{A} \right\|_{\text{F}(\Omega)} = \sqrt{\sum_{(i, j)\in \Omega}a_{ij}^2}$.

\subsection{Constrained block-wise optimization}
By integrating the reparameterization (\ref{equ:repara}) into the quadratic majorization \eqref{equ:quad-maj}, each MM update reduces to solving a constrained matrix factorization problem
\begin{equation}\label{equ:MM-UV}
   \min_{\mathbf{U}=(\boldsymbol \theta, \boldsymbol 1), \mathbf{V}=(\mathbf{a}, \mathbf{b})} \|\mathbf{U}\mathbf{V}^\top - \tilde{\mathbf{Y}}\|_{\text{F}(\Omega)}^2, ~~\text{subject to} ~~ \mathbf{a} \geq \mathbf{0}.
\end{equation}

Observing that the objective function in (\ref{equ:MM-UV}) is quadratic in $\boldsymbol \theta$, $\mathbf{a}$, and $\mathbf{b}$, we alternately optimize over these variables. Given $\hat{\boldsymbol \theta}$ and $\hat{\mathbf{b}}$ from the previous iteration, updating $\mathbf{a}$ reduces to a standard entrywise non-negative least squares (NNLS) problem
\[\min_{\mathbf{a}\geq\mathbf{0}}\|\hat{\boldsymbol \theta}\mathbf{a}^\top - (\tilde{\mathbf{Y}}-\mathbf{1}\hat{\mathbf{b}}^\top)\|_{\text{F}(\Omega)}^2,
\]
which we solve using the active set method \citep{NNLS}. With the updated $\hat{\mathbf{a}}$, the parameter $\mathbf{b}$ admits a closed-form update $\hat{\mathbf{b}}=(\tilde{\mathbf{Y}}-\hat{\boldsymbol \theta}\hat{\mathbf{a}}^\top)^\top\mathbf{1}/N$. Subsequently, given $\hat{\mathbf{a}}$ and $\hat{\mathbf{b}}$, the update for $\boldsymbol{\theta}$ is obtained in closed form as $\hat{\boldsymbol{\theta}}= (\tilde{\mathbf{Y}}-\mathbf{1}\hat{\mathbf{b}}^\top)\hat{\mathbf{a}}/(\hat{\mathbf{a}}^{\top}\hat{\mathbf{a}})$. Detailed derivations are omitted for brevity.
We refer to the proposed computational framework as constrained block MM (\texttt{cBMM}), and summarize the full procedure in Algorithm \ref{alg:cBMM}. Per-iteration computational complexity is provided in the Appendices \ref{supp:complexity}.

\begin{algorithm}[t]
\caption{Constrained Block MM (\texttt{cBMM})}
\label{alg:cBMM}
\begin{algorithmic}[1]
\setcounter{ALG@line}{0}
    \State \underline{Input}: $\Omega$, $\mathbf{Y}$, tolerance \texttt{tol}
    \State \underline{Output}: $\boldsymbol\theta$, $\mathbf{a}$, $\mathbf{b}$
    \State Initialize $\mathbf{U} = (\mathbf{u},\mathbf{1})$, $\mathbf{V} = (\mathbf{v}_1,\mathbf{v}_2)$, set $\delta = \infty$
    \While{$\delta > \texttt{tol}$}
        \State $\tilde{\mathbf{X}} \leftarrow \mathbf{u}\mathbf{v}_1^\top + \mathbf{1}\mathbf{v}_2^\top$
        \State $\tilde{\mathbf{Y}} \leftarrow \tilde{\mathbf{X}} + 4\sigma\mathbf{Y}\circ\Phi(-\mathbf{Y}\circ\tilde{\mathbf{X}})$

        \State Compute $\mathbf{R} \leftarrow \tilde{\mathbf{Y}}-\mathbf{1}\mathbf{v}_2^\top$ and $\mathbf{v}_1 \leftarrow \arg\min_{\mathbf{v}_1\geq\mathbf{0}}\|\mathbf{u}\mathbf{v}_1^\top - \mathbf{R}\|_{\text{F}(\Omega)}^2 $

        \State Update $\mathbf{v}_2 \leftarrow (\tilde{\mathbf{Y}}-\mathbf{u}\mathbf{v}_1^\top)^\top\mathbf{1}/N$

        \State Update $\mathbf{R} \leftarrow \tilde{\mathbf{Y}}-\mathbf{1}\mathbf{v}_2^\top$ and $\mathbf{u} \leftarrow \mathbf{R}\mathbf{v}_1 (\mathbf{v}_1^\top\mathbf{v}_1)^{-1}$

        \State $\delta \leftarrow |\ell(\mathbf{u}\mathbf{v}_1^\top + \mathbf{1}\mathbf{v}_2^\top) -\ell(\tilde{\mathbf{X}})|/|\ell(\tilde{\mathbf{X}})|$
    \EndWhile
    \State \Return $\boldsymbol\theta = \mathbf{u}$,\quad $\mathbf{a} = \mathbf{v}_1$,\quad $\mathbf{b} = \mathbf{v}_2$
\end{algorithmic}
\end{algorithm}

\subsection{Identifiability}\label{subsec-iden}
As noted in \cite{irt1} and \cite{lart}, the estimation of $\boldsymbol\theta$ and $\mathbf{a}$ under reparameterization (\ref{equ:repara}) is subject to identifiability issues, such as trivial solutions and symmetries involving sign-flips and scaling, leading to multiple optima with an identical score matrix $\mathbf{X}$. To resolve these ambiguities, specific structural constraints and architectural considerations must be imposed.
We begin by providing the formal definition of identifiability.

\begin{definition}[Identifiability]\label{def:iden}
The latent probabilistic model is identifiable if for any solution $(\boldsymbol\theta^*, \mathbf{a}^*, \mathbf{b}^*)$, another suite of $(\boldsymbol\theta^\dagger, \mathbf{a}^\dagger, \mathbf{b}^\dagger)$ that induces the same distribution matrix, i.e., $\Phi(\mathbf{X}^\dagger) = \Phi(\mathbf{X}^*)$, must satisfy $(\boldsymbol\theta^\dagger, \mathbf{a}^\dagger, \mathbf{b}^\dagger)=(\boldsymbol\theta^*, \mathbf{a}^*, \mathbf{b}^*)$.
\end{definition}

Note that strict identifiability based on Definition \ref{def:iden} is not generally guaranteed, as hinted by the ambiguities mentioned previously. We show that, however, the model is identifiable up to an equivalence class, as formalized in  Definition \ref{def:eqclass}, with the underlying order structure preserved. 

\begin{definition}[Equivalence class]\label{def:eqclass}
Solution sets $(\boldsymbol\theta^*, \mathbf{a}^*, \mathbf{b}^*)$ and $(\boldsymbol\theta^\dagger, \mathbf{a}^\dagger, \mathbf{b}^\dagger)$ are within an equivalence class if  each corresponding pair, $(\boldsymbol\theta^*, \boldsymbol\theta^\dagger)$,  $(\mathbf{a}^*, \mathbf{a}^\dagger)$, and $(\mathbf{b}^*, \mathbf{b}^\dagger)$, induces the same rank ordering.
\end{definition}

With all definitions in place, we are ready to present the following theorem to characterize identifiability of (\ref{equ:repara}).
We leave the full proof to the Appendices \ref{supp:proof1}.
\begin{theorem}[Structural preservation]\label{thm:iden}
The latent probabilistic model (\ref{equ:logit}) is identifiable up to an equivalence class if $\boldsymbol\theta\neq\mathbf{c}$, where $\mathbf{c}$ is a constant vector, and  $\mathbf{a} \geq \mathbf{0}$ with $\mathbf{a} \neq \mathbf{0}$.
\end{theorem}


\subsection{Convergence}
We establish convergence guarantees for \texttt{cBMM} under the following assumptions.

\begin{itemize}
    \item[(H1)] \textbf{Boundedness.} The iterates $\{(\boldsymbol{\theta}^{(t)}, \mathbf{a}^{(t)}, \mathbf{b}^{(t)})\}$ are contained in a compact set $\mathcal{K}$, with $\|\mathbf{W}\| \leq M$ for all $\mathbf{W} \in \mathcal{K}$ and some constant $M > 0$.
    \item[(H2)] \textbf{Non-degeneracy.} There exists a constant $c > 0$ such that $\|\boldsymbol{\theta}^{(t)}\| \geq c$ and $\|\mathbf{a}^{(t)}\| \geq c$ for all $t \geq 0$.
\end{itemize}

Assumption (H1) is standard for iterative optimization algorithms and can be enforced by adding a bounding constraint to the feasible set. Assumption (H2) excludes degenerate solutions with $\boldsymbol{\theta} = \mathbf{0}$ or $\mathbf{a} = \mathbf{0}$, which are already ruled out by the non-triviality condition in Theorem \ref{thm:iden}. Let $\mathbf{W} = (\boldsymbol{\theta}, \mathbf{a}, \mathbf{b}) \in \mathbb{R}^{N + 2J}$ denote the concatenated parameter vector, and define $F(\mathbf{W}) \equiv \ell(\boldsymbol{\theta}\mathbf{a}^\top + \mathbf{1}\mathbf{b}^\top)$ as the objective viewed as a function of $\mathbf{W}$.

With assumptions in place, we present the main convergence result in Theorem \ref{thm:conv}, established under the Kurdyka--\L{}ojasiewicz (KL) framework \citep{KL1, KL2, KL3}. The complete proof is provided in the Appendices \ref{supp:proof2}.

\begin{theorem}[Convergence of \texttt{cBMM}]\label{thm:conv}
Under Assumptions \textnormal{(H1)} and \textnormal{(H2)}, the sequence $\{(\boldsymbol{\theta}^{(t)}, \mathbf{a}^{(t)}, \mathbf{b}^{(t)})\}$ generated by Algorithm \ref{alg:cBMM} satisfies:
\begin{enumerate}
    \item[(i)] \textnormal{(Summability)} $\displaystyle\sum_{t=0}^{\infty}\|\mathbf{W}^{(t+1)} - \mathbf{W}^{(t)}\| < \infty$;
    \item[(ii)] \textnormal{(Whole-sequence convergence)} $\mathbf{W}^{(t)} \to \mathbf{W}^* = (\boldsymbol{\theta}^*, \mathbf{a}^*, \mathbf{b}^*)$ as $t \to \infty$;
    \item[(iii)] \textnormal{(Stationarity)} $\mathbf{W}^*$ satisfies the KKT conditions:
    \[
    \nabla_{\boldsymbol{\theta}}\ell(\mathbf{X}^*) = \mathbf{0}, \quad \mathbf{a}^* \geq \mathbf{0}, \quad \nabla_{\mathbf{a}}\ell(\mathbf{X}^*) \geq \mathbf{0}, \quad \mathbf{a}^* \circ \nabla_{\mathbf{a}}\ell(\mathbf{X}^*) = \mathbf{0}.
    \]
\end{enumerate}
Furthermore, when the KL exponent is $1/2$, the convergence is \textbf{linear}, i.e., there exist constants $C > 0$ and $\alpha \in (0,1)$ such that $\|\mathbf{W}^{(t)} - \mathbf{W}^*\| \leq C\alpha^t$.
\end{theorem}

\section{Experiments}\label{sec-exp}
We conduct extensive experiments using both synthetic data and real-world applications. On the one hand, simulation studies on synthetic data are used for sensitivity analysis under varying problem parameters, including the temperature, the missing pattern and rate, and the sparsity level of the discrimination parameter; Detailed results are provided in Appendices \ref{supp:sensitivity}, with key findings summarized below.
On the other hand, we evaluate \texttt{cBMM} on real-world benchmark suites to assess scalability by comparing its runtime against baseline methods, and to assess interpretability through the recovered model ability, item difficulty and discrimination.

\textbf{Baseline methods.}\quad
We compare \texttt{cBMM} with several baselines. First, we consider conventional IRT methods with two implementations: (i) \texttt{mirt} in \texttt{R} \citep{mirt}, which employs the EM algorithm for parameter estimation, and (ii) \texttt{py-irt} in Python \citep{pyirt}, which uses stochastic variational inference \citep{svi} with mini-batch gradient updates. We also include generic optimization approaches for solving (\ref{equ:loss}). Specifically, we consider Limited-memory BFGS with Box constraints (\texttt{L-BFGS-B}) \citep{LBFGSB}, a second-order quasi-Newton method with bounded constraints, and \texttt{Manopt} \citep{Manopt}, a framework for optimization on manifolds. Implementation details are provided in Appendices \ref{supp:algdetail}.

\textbf{Key findings in simulation studies.}\quad
As shown in Appendices \ref{supp:sensitivity}, simulation studies on synthetic data reveal several key findings: (i) \texttt{cBMM} demonstrates robust performance under varying temperatures, missing patterns and rates, and sparsity levels of the discrimination parameter $\mathbf{a}$, consistently achieving comparable or even superior estimation accuracy while running significantly faster (up to 200x) than competing methods.
(ii) \texttt{py-irt} \textit{produces substantially lower estimation accuracy than the other methods, and is therefore excluded from the real-world benchmark application studies.}

\textbf{Real-world benchmark suites.}\quad
We consider multiple benchmark suites with varying numbers of LLMs ($N$) and benchmark items ($J$) to assess scalability and interpretability. In total, we examine seven benchmark suites. The first, MATH-500 \citep{MATH500}, contains $J=500$ items  and is evaluated across $N=140$ LLMs; we adopt the evaluation data collected by \cite{lart}, with full model lists and prompt details provided in the Appendices \ref{supp:benchmarkdetail}. The remaining six benchmark suites are drawn from the Hugging Face Open LLM Leaderboard \citep{open-llm-leaderboard}, which includes $N=2211$ LLMs evaluated between June and December 2024 under the Leaderboard v2 schema. These suites vary in the number of benchmark items ($J$): IFEval (541), MuSR (756), GPQA ($1192$), MATH (1324),  BBH (5761),  and MMLU-Pro (12032) \citep{IFEval, MUSR, GPQA, MATH, BBH, MMLU}. We use the raw historical evaluation data obtained via the \texttt{huggingface\_hub} API as curated by \cite{wu2026}, with further details available in the Appendices \ref{supp:benchmarkdetail}.

\subsection{Scalability}\label{subsec-scale}
To demonstrate the scalability of \texttt{cBMM} compared to other baseline methods, we evaluate it on six large-scale benchmark suites from the Hugging Face Open LLM Leaderboard. For each suite, we conduct 50 replicates and report the average, minimum and maximum runtimes (in seconds) along with the average speedup ratio of \texttt{cBMM} relative to the fastest baseline.

\textbf{Significant speedup and superior scalability.}\quad
As illustrated in Figure \ref{app:runtime}, \texttt{cBMM} consistently outperforms existing baselines in computational efficiency, achieving a 41x-86x reduction in runtime. As the number of benchmark items $J$ increases, \texttt{cBMM} exhibits improved scalability compared to generic optimization baselines. For instance, scaling from the IFEval suite ($J=541$) to the MMLU-Pro suite ($J=12032$) results in an average 508x increase in runtime for \texttt{Manopt}, a 392x increase in runtime for \texttt{L-BFGS-B}, but only a 104x increase for \texttt{cBMM}. Notably, \texttt{mirt} suffers from severe numerical instability and fails to converge (indicated by `NA') on four large benchmark suites (GPQA, MATH, BBH, and MMLU-Pro), highlighting its limitations in large-scale settings. As a complement, our simulation studies in the Appendices \ref{supp:sensitivity} further demonstrate that \texttt{cBMM} can achieve up to a 200x reduction in runtime while maintaining estimation accuracy comparable to, or even exceeding, that of baseline methods.

\begin{figure}[t]
 \centering
 \includegraphics[width=\linewidth]{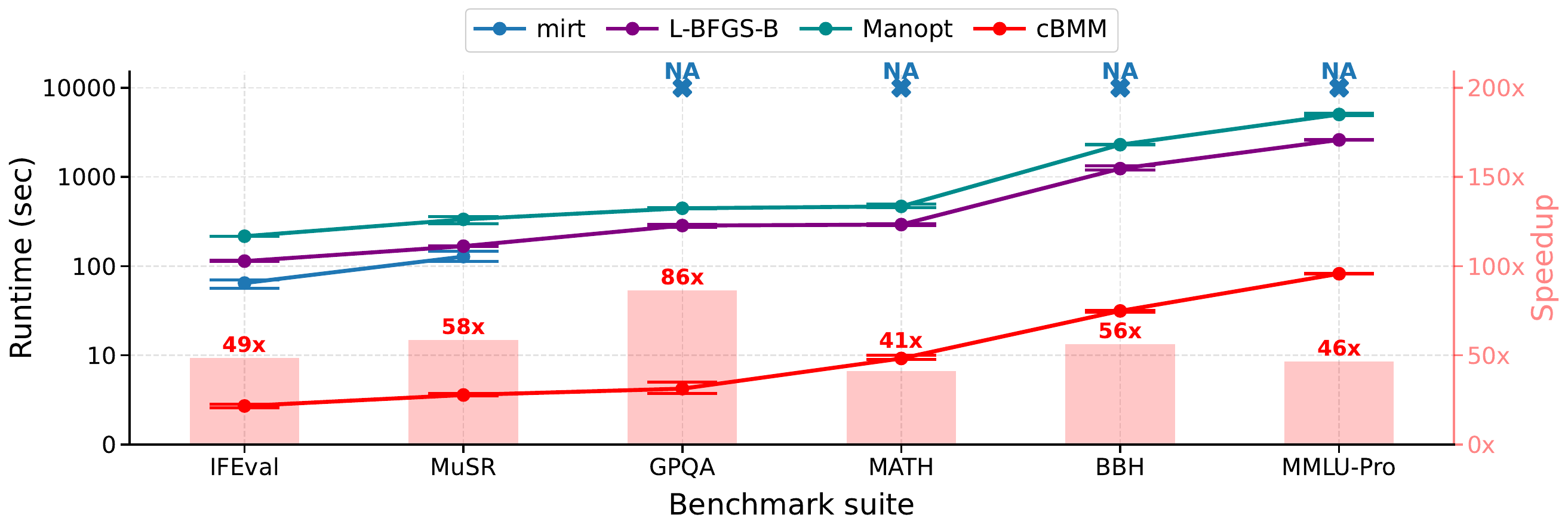}
 \caption{Scalability comparison across six large-scale benchmark suites. Runtime (sec) (left axis, points and interval markers) is shown on a logarithmic scale and reports the average runtime (points) as well as the minimum and maximum runtime (interval markers) across 50 replicates for each method. Speedup (right axis, shaded bars) represents the average speedup  of \texttt{cBMM} over the fastest baseline. Benchmark suites are ordered by increasing scale of $J$ with fixed $N=2211$. Exact speedup ratios are annotated for each suite. `NA' indicates convergence failure due to ill-conditioned data configurations in \texttt{mirt}.}\label{app:runtime}
\end{figure}

\subsection{Interpretability of model ability}
To interpret and validate the \texttt{cBMM}-recovered model abilities, we employ the MATH-500 benchmark suite, which encompasses a diverse set of LLMs released between July 2023 and September 2025 with varying model sizes. To facilitate comparison, we categorize these models into three scale tiers: small (0–7B), medium (7–20B), and large (>20B).

\textbf{Consistency with parametric scaling laws.}\quad
As shown in Figure \ref{app:math500theta}, the \texttt{cBMM}-recovered model abilities $\hat{\boldsymbol{\theta}}$ exhibit a clear upward trend over release time across all three model scale tiers, consistent with empirical expectations. In addition,  the parametric scaling effect, in line with \cite{scalinglaw}, becomes evident: larger models exhibit stronger reasoning capabilities, as reflected by higher overall placements and peak values within each tier. A notable deviation is observed for the 01-ai/Yi-34B model, which is primarily driven by the disproportionate scaling of training tokens relative to model parameters within the Yi series \citep{Yi}. This phenomenon is consistent with the Chinchilla scaling laws \citep{scalinglaw2}, which highlight the importance of balancing training data and model parameters for optimal model performance.

\begin{figure}[t]
 \centering
 \includegraphics[width=\linewidth]{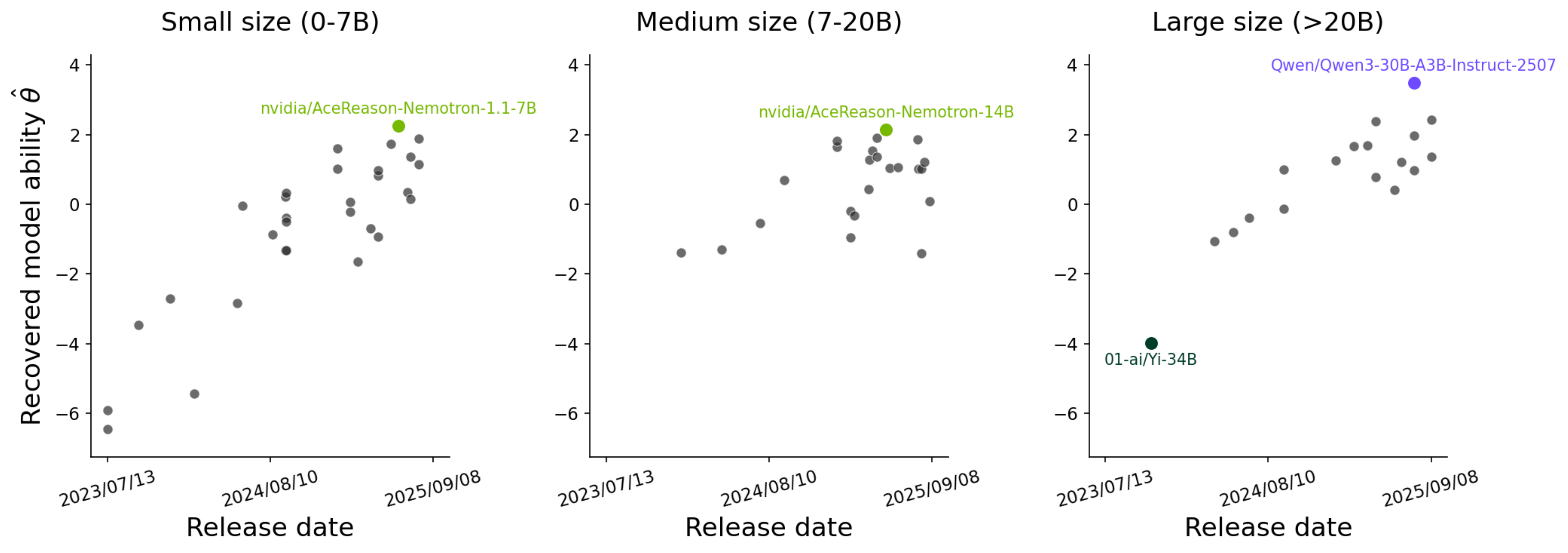}
 \caption{\texttt{cBMM}-recovered model abilities $\hat{\boldsymbol{\theta}}$ for LLMs in the MATH-500 benchmark suite. The three subplots (left, middle, right) correspond to three model scale tiers: small (0–7B), medium (7–20B), and large (>20B). All panels span the same release date range, from July 2023 to September 2025. LLMs with the highest $\hat{\boldsymbol{\theta}}$ are highlighted in each subplot, along with the LLM with the lowest $\hat{\boldsymbol{\theta}}$ in the large-size subplot.}\label{app:math500theta}
\end{figure}

\subsection{Interpretability of item difficulty}
To interpret and validate the \texttt{cBMM}-recovered item difficulties, we use the human-annotated difficulty levels for each item in the MATH-500 benchmark suite as reference and assess the effectiveness of the estimated parameters $\hat{\mathbf{b}}$ obtained from \texttt{cBMM}.

\textbf{Consistency with human-annotated difficulty.}\quad
Figure \ref{app:math500b}, with subplots corresponding to overall and subject-wise results, displays the distribution of  \texttt{cBMM}-recovered difficulty estimates ($-\hat{\mathbf{b}}$) across human-annotated difficulty levels (1–5) for the MATH-500 benchmark suite.
A generally consistent upward trend is observed across all  subplots at the median level, aligning well with the ordinal human-annotated difficulty scale. The `Overall' subplot shows increasing variability in the distribution of ($-\hat{\mathbf{b}}$)  as the difficulty level rises. We also observe several deviations from the overall increasing trend,  such as level 1 in `Geometry' and level 3 in `Intermediate Algebra'. These discrepancies likely reflect a latent misalignment between model-perceived and human-judged difficulty, particularly for more challenging items. This observation is consistent with concerns raised by \cite{relatedwork6} and \cite{relatedwork7}, which suggest that benchmark datasets may lack sensitivity and specificity due to confounding features that influence LLM-perceived difficulty in ways not readily aligned with human intuition.


\begin{figure}[t]
 \centering
 \includegraphics[width=\linewidth]{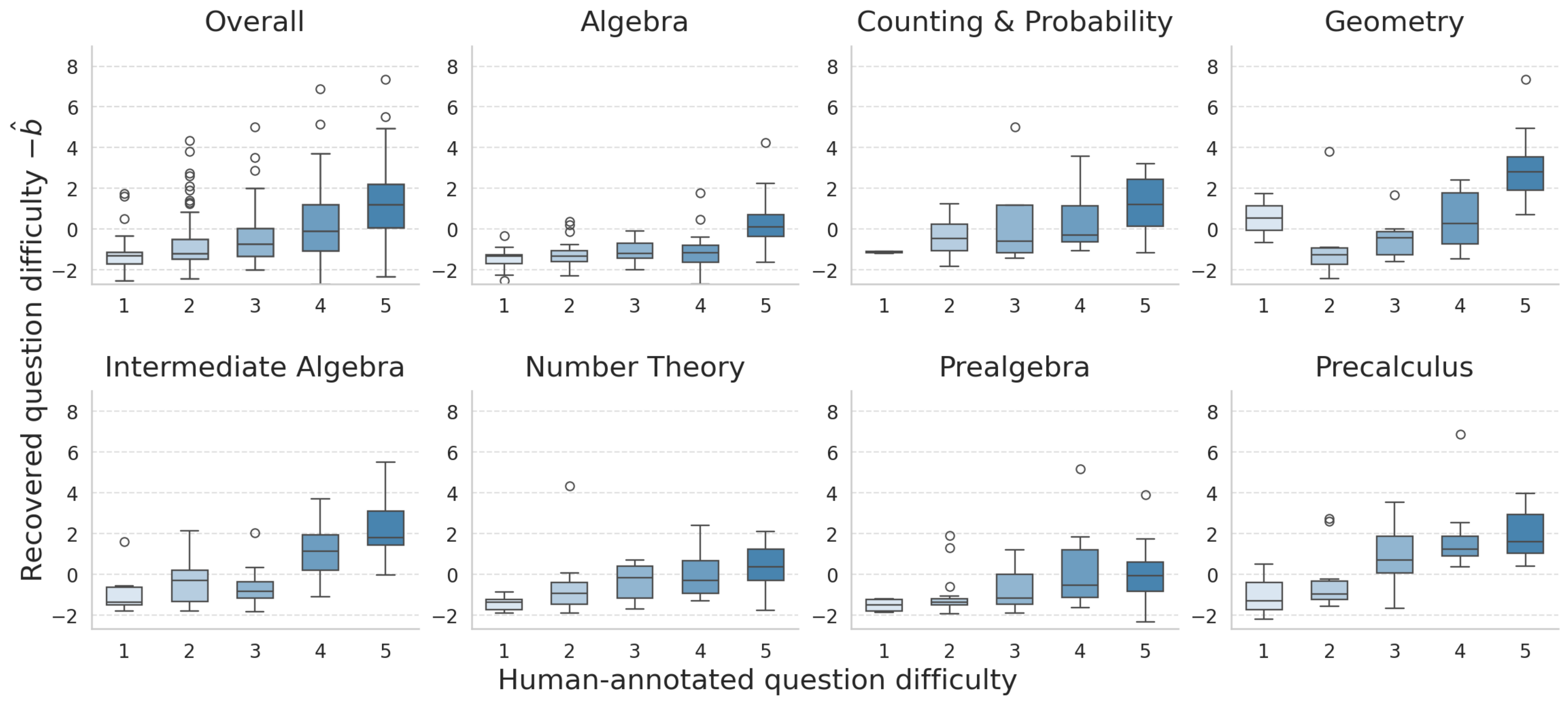}
 \caption{Distribution of \texttt{cBMM}-recovered item difficulties ($-\hat{\mathbf{b}}$) across human-annotated difficulty levels (1-5) in the MATH-500 benchmark suite, shown for the overall and subject-specific results. 
 }
 \label{app:math500b}
\end{figure}

\subsection{Interpretability of item discrimination}

Regarding item discrimination, we hypothesize that greater heterogeneity in discrimination estimates leads to lower ranking consistency between the \texttt{cBMM}-recovered model abilities and average accuracies, since average accuracy implicitly assumes equal item discrimination.
Specifically, we quantify ranking consistency using the Spearman correlation between the \texttt{cBMM}-recovered $\hat{\boldsymbol{\theta}}$ and average accuracies, a standard metric in IRT-based LLM evaluations \citep{irt3,relatedwork7}. To measure heterogeneity in item discrimination, we use the coefficient of variation (CV) of the \texttt{cBMM}-recovered $\hat{\mathbf{a}}$, defined as the ratio of the standard deviation to the mean.

\textbf{Inverse relationship between ranking consistency and discrimination heterogeneity.}\quad
Figure \ref{app:hfscatter} depicts the relationship between the Spearman correlation of \texttt{cBMM}-recovered $\hat{\boldsymbol\theta}$ and average accuracies and the CV of \texttt{cBMM}-recovered $\hat{\mathbf{a}}$, across seven benchmark suites.
To highlight different analysis scopes, we present results for the top 20\% of models selected by average accuracy (left panel) and for the full set of LLMs (right panel). A clear inverse relationship is observed, providing empirical support for our hypothesis that greater heterogeneity in item discrimination estimates leads to more pronounced ranking discrepancies. A similar trend has also been reported in \cite{irt3}.


\begin{figure}[t]
 \centering
 \includegraphics[width=\linewidth]{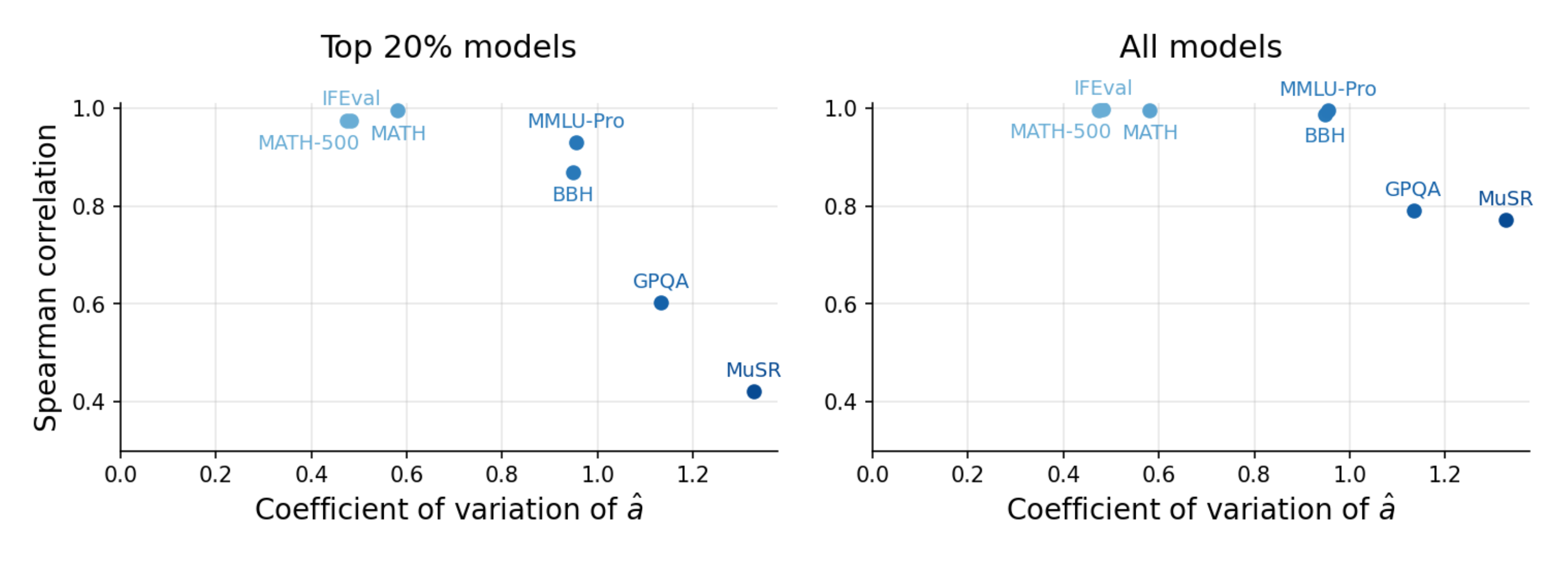}
 \caption{Spearman correlation between \texttt{cBMM}-recovered model abilities and average accuracies v.s. the coefficient of variation of  \texttt{cBMM}-recovered $\hat{\mathbf{a}}$. The left panel shows results for the top 20\% of models selected by average accuracy, while the right panel includes all LLMs. Each scatter point corresponds to a benchmark suite, as annotated in the plot.}\label{app:hfscatter}
\end{figure}

\textbf{Fine-grained ranking shifts under discrimination heterogeneity}\quad
Figure \ref{app:hfrankflow} illustrates the detailed ranking flow between average accuracy and \texttt{cBMM}-recovered ability across three representative benchmark suites (MATH, MMLU-Pro, and GPQA). We select the top 30 models for MATH, top 50 for MMLU-Pro, and top 90 for GPQA, and visualize ranking transitions for LLMs that appear in both ranking schemes; full rankings are provided in the Supplementary Material.
These ranking flows visually demonstrate that ranking discrepancies become increasingly pronounced as item discrimination grows more heterogeneous, offering a fine-grained view of ranking inconsistency. Additionally, several notable patterns emerge from the MATH benchmark: rankings for distilled LLMs from the Qwen series \citep{qwen2.5} remain largely consistent across both metrics, whereas substantial upward shifts are observed for Qwen2-Math-72B-Instruct and Qwen2.5-72B when transitioning from average accuracy to \texttt{cBMM}-recovered ability, suggesting that their capabilities may be underestimated by standard accuracy-based evaluations.

\begin{figure}[t]
 \centering
 \includegraphics[width=\linewidth]{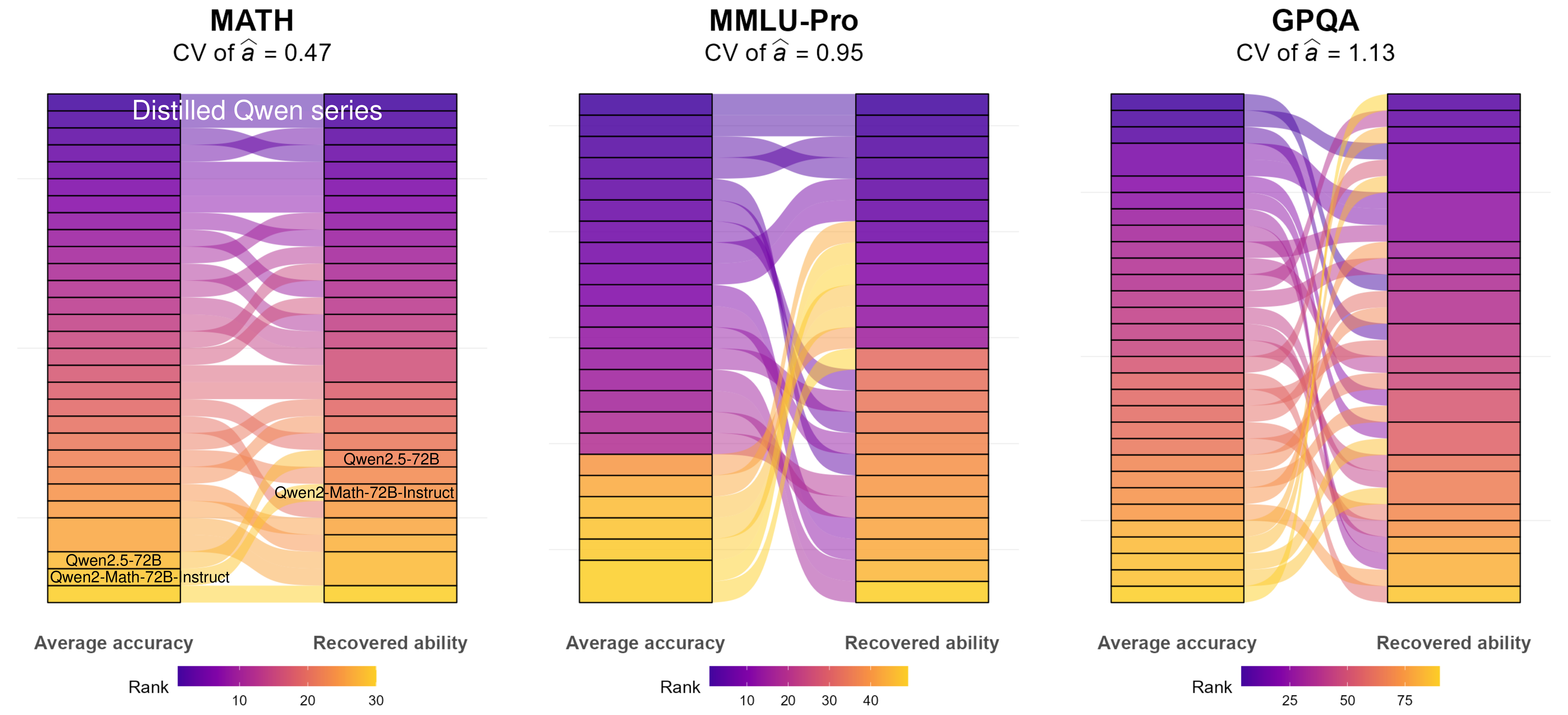}
 \caption{Ranking flow between average accuracy (left column in each panel) and \texttt{cBMM}-recovered ability (right column in each panel). The three subplots correspond to three  benchmark suites: MATH, MMLU-Pro and GPQA, ordered by increasing coefficient of variation of the recovered $\hat{\mathbf{a}}$. Only LLMs that consistently appear in the top tiers under both ranking schemes are shown; full rankings are provided in the Supplementary Material. Wider bars indicate tied rankings.}\label{app:hfrankflow}
\end{figure}



\section{Conclusion}\label{sec-conc}
This paper proposes an interpretable and scalable computational framework for evaluating LLMs, providing fine-grained assessments of model abilities and enabling large-scale real-world applications. Building on classical IRT and leveraging the MM principle, we reformulate the evaluation problem as a sequence of constrained matrix factorization problems and introduce the \texttt{cBMM} algorithm to solve it efficiently. We further establish theoretical guarantees for parameter identifiability and algorithmic convergence. Through extensive experiments on both synthetic data and large-scale real-world benchmarks, we demonstrate that \texttt{cBMM} achieves orders-of-magnitude computational speedups and improved scalability while maintaining competitive estimation accuracy relative to existing methods. Moreover, its evaluation results are consistent with established parametric scaling laws \citep{scalinglaw,scalinglaw2} and human annotations, while revealing heterogeneity in model and item behavior that is not captured by average-accuracy-based evaluation.

\textbf{Limitation and future work.}\quad
Future work is motivated by a few current limitations. First, we treat model capability as a single dimensional trait, whereas general benchmarks, such as BBH and MMLU-Pro, may reveal multi-faceted model abilities. This motivates a natural extension to multidimensional IRT frameworks. We note that a corresponding implementation is already available in our \href{https://anonymous.4open.science/r/cBMM-8E04}{GitHub repository}, although additional theoretical challenges remain, particularly regarding identifiability under rotational invariance.
Second, the logistic link in (\ref{equ:logit}) can be extended to accommodate ordinal responses, such as those derived from LLM-as-a-Judge evaluations \citep{LLMJudge}.
Third, deriving statistical error bounds for the recovered parameters would provide a principled quantification of uncertainty and strengthen the theoretical guarantees of the method.

\textbf{Impact statement.}\quad
Our findings from large-scale evaluation tasks have important implications for fairer LLM capability ranking and more principled benchmark design.
First, by accounting for LLM stochasticity and item heterogeneity, the recovered parameters capture the interaction between model abilities and item characteristics, leading to more reliable capability ranking of LLMs.
Second, sparse discrimination estimates can identify benchmark items that are non-informative for distinguishing LLM abilities, providing a data-driven signal of redundancy and guiding more efficient benchmark construction. Moreover, the recovered item difficulty parameters offer data-driven assessments that complement, and potentially augment, human annotations, enabling more automated and scalable benchmark analysis for LLM evaluation.

\clearpage
\bibliographystyle{plain}
\bibliography{bibliography}


\clearpage
\appendix
\renewcommand{\thetable}{\thesection.\arabic{table}}
\setcounter{table}{0}
\setcounter{equation}{0}
\renewcommand{\theequation}{\thesection.\arabic{equation}}
\setcounter{figure}{0}
\renewcommand{\thefigure}{\thesection.\arabic{figure}}

\section{Appendices}

\subsection{Related work}
\textbf{1-bit matrix completion.}\quad
The optimization problem underlying LLM evaluation based on IRT can be viewed as a constrained 1-bit matrix completion problem with a user-specified rank. The 1-bit matrix completion problem has been well studied in the literature \citep{hdist1,hdist2,relatedwork3, MMGN}, yet its computation remains challenging. \texttt{cBMM} builds on the MMGN framework \citep{MMGN}, a recently proposed efficient algorithm for 1-bit matrix completion based on the MM principle. \texttt{cBMM} is specifically designed for the constrained formulation arising in LLM evaluation, and reformulates the original problem as a sequence of constrained matrix factorization subproblems. This leads to closed-form updates and a simple implementation based on standard linear algebra operations. As a result, \texttt{cBMM} is substantially faster than both IRT-based approaches and generic optimization methods, including quasi-Newton and manifold optimization algorithms.

\textbf{IRT-based LLM evaluation.}\quad
Conventional LLM evaluation relies on accuracy-based benchmarking in an aggregated manner \citep{swebenchverified,GAIA}, ignoring both the stochastic nature of LLM outputs and the heterogeneity of benchmark items \citep{intro11,redun2}. Recently, methods from psychometrics, particularly Item Response Theory (IRT), have been adapted for LLM evaluation \citep{relatedwork31,relatedwork32,intro17,irt3}, using a latent probabilistic framework to model LLM ability and item characteristics. Existing studies typically build on the two-parameter IRT model, whose conventional implementations often suffer from numerical instability and computational inefficiency \citep{mirt,lart}, thereby limiting their applicability to small-scale settings. \texttt{cBMM} addresses these challenges by introducing a scalable and interpretable computational framework based on MM optimization. By ensuring both computational efficiency and stable parameter recovery, \texttt{cBMM} is well-suited for large-scale LLM evaluations.


\subsection{Proof of Theorem \ref{thm:iden}}\label{supp:proof1}

Suppose a triplet $(\boldsymbol{\theta}, \mathbf{a}, \mathbf{b})$ satisfies $\boldsymbol{\theta} \ne \mathbf{c}$ for some
constant vector $\mathbf{c}$, and that $\bf a$ is nonnegative with at least one nonzero entry
(without loss of generality, assume $a_1 > 0$). Consider a given score matrix
$\mathbf{X} = \boldsymbol\theta \mathbf{a}^\top + \mathbf{1}\mathbf{b}^\top$.
First, suppose that there exists a column $\mathbf{X}_{,j}$ that is neither a scalar multiple
of the first column $\mathbf{X}_{,1}$ nor the vector of ones. Without loss of generality, suppose
that this column is $\mathbf{X}_{,2}$. We argue that for any $i,j,k$,
\[\frac{x_{i1} - x_{j1}}{x_{j1} - x_{k1}} = \frac{x_{i2} -
x_{j2}}{x_{j2} - x_{k2}}.\] That is, the scaled differences among
entries must be consistent. This is because the expression reduces to
$\frac{(\theta_i -\theta_j)a_1}{(\theta_j -\theta_k)a_1} =
\frac{(\theta_i -\theta_j)a_2}{(\theta_j -\theta_k)a_2}$. We can now
identify $\boldsymbol{\theta}$ up to a scaling constant by first computing $\alpha
\equiv \frac{a_1}{a_2} = \frac{x_{i1} - x_{j1}}{x_{i2} - x_{j2}}$
where $x_{i1}$ and $x_{j1}$ are distinct, which is guaranteed under the assumptions. We then solve
\[\mathbf{X}_{,1} - b_1{\bf 1} = \alpha (\mathbf{X}_{,2}
- b_2{\bf1})\]
to find $b_1$ and $b_2$. This further allows us to obtain $ a_1\boldsymbol{\theta}$ through
$a_1\boldsymbol{\theta}=\mathbf{X}_{,1} - b_1{\bf 1}$. Since $a_1 > 0$, the rank ordering of $\boldsymbol{\theta}$ is uniquely determined.

Now consider any two non-identical rows $i$ and $j$ of $\mathbf{X}$. We have
$\mathbf{X}_{i,} - \mathbf{X}_{j,} = (\theta_i - \theta_j) \bf a$. Hence, $\mathbf{a}$ can be identified up to a scaling constant through
\[\frac{a_k}{a_1} = \frac{x_{ik} - x_{jk}}{x_{i1} - x_{j1}}.\]
In particular, the rank ordering of $\bf a$ is determined. Finally, because
$\boldsymbol{\theta} \mathbf{a}^\top$ is uniquely determined, $\bf b$ is also uniquely determined.

Now consider the case that every column aside from $\mathbf{X}_{,1}$ is either
a scalar multiple of $\mathbf{X}_{,1}$ or the vector of ones. First note that any
column $i$ which is a scalar multiple of the vector of ones identifies a
column where $a_i$ is zero, and the corresponding entry of $\mathbf{b}$ is given by the value in that column. For any column that is a scalar multiple of $\mathbf{X}_{,1}$, the previous arguments still hold to uniquely
identify the rank ordering for $\bf a$ and $\boldsymbol \theta$, except that $\boldsymbol b$ is no
longer uniquely determined, but rather follows the same ratios as the
corresponding columns. Nonetheless, the rank ordering of $(\boldsymbol{\theta}, \mathbf{a}, \mathbf{b})$
is uniquely determined.

\subsection{Proof of Theorem \ref{thm:conv}}\label{supp:proof2}
We begin by establishing a sufficient descent condition, which is the key step for applying the Kurdyka--\L{}ojasiewicz (KL) framework \citep{KL1, KL2, KL3}.

\begin{lemma}[Sufficient Descent]\label{lem:descent}
Under Assumptions \textnormal{(H1)} and \textnormal{(H2)}, there exists a constant $\rho > 0$ such that for all $t \geq 0$,
\[
F(\mathbf{W}^{(t)}) - F(\mathbf{W}^{(t+1)}) \geq \rho \|\mathbf{W}^{(t+1)} - \mathbf{W}^{(t)}\|^2.
\]
\end{lemma}

\begin{proof}
Fix iteration $t$. Let $\tilde{\mathbf{X}}^{(t)} = \boldsymbol{\theta}^{(t)}(\mathbf{a}^{(t)})^\top + \mathbf{1}(\mathbf{b}^{(t)})^\top$ and $\tilde{\mathbf{Y}}^{(t)} = \tilde{\mathbf{X}}^{(t)} + 4\sigma\mathbf{Y} \circ \Phi(-\mathbf{Y} \circ \tilde{\mathbf{X}}^{(t)})$ as defined in Algorithm \ref{alg:cBMM}. Define the surrogate function at step $t$ as $G(\mathbf{W} \mid \mathbf{W}^{(t)}) \equiv g(\boldsymbol{\theta}\mathbf{a}^\top + \mathbf{1}\mathbf{b}^\top \mid \tilde{\mathbf{X}}^{(t)})$.

We analyze each block update in Algorithm \ref{alg:cBMM}.

\textit{Step 1 (Update $\mathbf{a}$, line 7).} With $\boldsymbol{\theta}^{(t)}$ and $\mathbf{b}^{(t)}$ fixed, the sub-objective reduces to $\sum_j \|\theta^{(t)} a_j - R_{1,\cdot j}\|^2$, which is strongly convex in $\mathbf{a}$ with modulus $\mu_{\mathbf{a}} = \|\boldsymbol{\theta}^{(t)}\|^2 \geq c^2$ by Assumption (H2). By the strong convexity optimality inequality,
\begin{equation}\label{equ:descent_a}
G(\mathbf{a}^{(t)} \mid \mathbf{W}^{(t)}) - G(\mathbf{a}^{(t+1)} \mid \mathbf{W}^{(t)}) \geq c^2\|\mathbf{a}^{(t+1)} - \mathbf{a}^{(t)}\|^2.
\end{equation}

\textit{Step 2 (Update $\mathbf{b}$, line 8).} With $\boldsymbol{\theta}^{(t)}$ and $\mathbf{a}^{(t+1)}$ fixed, the sub-objective is strongly convex in $\mathbf{b}$ with modulus $\mu_{\mathbf{b}} = N$. Hence,
\begin{equation}\label{equ:descent_b}
G(\mathbf{b}^{(t)} \mid \mathbf{W}^{(t)}) - G(\mathbf{b}^{(t+1)} \mid \mathbf{W}^{(t)}) \geq N\|\mathbf{b}^{(t+1)} - \mathbf{b}^{(t)}\|^2.
\end{equation}

\textit{Step 3 (Update $\boldsymbol{\theta}$, line 9).} With $\mathbf{a}^{(t+1)}$ and $\mathbf{b}^{(t+1)}$ fixed, the sub-objective is strongly convex in $\boldsymbol{\theta}$ with modulus $\mu_{\boldsymbol{\theta}} = \|\mathbf{a}^{(t+1)}\|^2 \geq c^2$ by Assumption (H2). Hence,
\begin{equation}\label{equ:descent_theta}
G(\boldsymbol{\theta}^{(t)} \mid \mathbf{W}^{(t)}) - G(\boldsymbol{\theta}^{(t+1)} \mid \mathbf{W}^{(t)}) \geq c^2\|\boldsymbol{\theta}^{(t+1)} - \boldsymbol{\theta}^{(t)}\|^2.
\end{equation}

Summing \eqref{equ:descent_a}--\eqref{equ:descent_theta} and letting $\rho_0 = \min(c^2, N) > 0$,
\[
G(\mathbf{W}^{(t)} \mid \mathbf{W}^{(t)}) - G(\mathbf{W}^{(t+1)} \mid \mathbf{W}^{(t)}) \geq \rho_0\|\mathbf{W}^{(t+1)} - \mathbf{W}^{(t)}\|^2.
\]
Applying the tangency condition $G(\mathbf{W}^{(t)} \mid \mathbf{W}^{(t)}) = F(\mathbf{W}^{(t)})$ and the domination condition $G(\mathbf{W}^{(t+1)} \mid \mathbf{W}^{(t)}) \geq F(\mathbf{W}^{(t+1)})$ yields the result with $\rho = \rho_0$.
\end{proof}

\begin{lemma}[Relative Error Condition]\label{lem:error}
Under Assumptions \textnormal{(H1)} and \textnormal{(H2)}, there exists a constant $\lambda > 0$ such that for all $t \geq 0$, there exists $\mathbf{s}^{(t+1)} \in \partial F(\mathbf{W}^{(t+1)})$ satisfying
\[
\|\mathbf{s}^{(t+1)}\| \leq \lambda\|\mathbf{W}^{(t+1)} - \mathbf{W}^{(t)}\|.
\]
\end{lemma}

\begin{proof}
We bound each component of the subgradient separately.

For $\boldsymbol{\theta}^{(t+1)}$ and $\mathbf{b}^{(t+1)}$, the corresponding block updates (lines 8--9 of Algorithm \ref{alg:cBMM}) are unconstrained least squares problems with exact closed-form solutions, so their first-order optimality conditions give $\nabla_{\boldsymbol{\theta}} G(\mathbf{W}^{(t+1)} \mid \mathbf{W}^{(t)}) = \mathbf{0}$ and $\nabla_{\mathbf{b}} G(\mathbf{W}^{(t+1)} \mid \mathbf{W}^{(t)}) = \mathbf{0}$. By the $L$-Lipschitz differentiability of $\ell$ and the double-linear bound $\|\mathbf{X}^{(t+1)} - \mathbf{X}^{(t)}\|_{\mathrm{F}} \leq M\|\mathbf{W}^{(t+1)} - \mathbf{W}^{(t)}\|$ under Assumption (H1),
\[
\|\nabla_{\boldsymbol{\theta}} F(\mathbf{W}^{(t+1)})\| \leq LM^2\|\mathbf{W}^{(t+1)} - \mathbf{W}^{(t)}\|,
\]
and similarly for $\nabla_{\mathbf{b}} F(\mathbf{W}^{(t+1)})$.

For $\mathbf{a}^{(t+1)}$, the NNLS update (line 7) satisfies the KKT conditions: there exists $\boldsymbol{\mu} \geq \mathbf{0}$ with $\nabla_{\mathbf{a}} G(\mathbf{W}^{(t+1)} \mid \mathbf{W}^{(t)}) = \boldsymbol{\mu}$ and $\boldsymbol{\mu} \circ \mathbf{a}^{(t+1)} = \mathbf{0}$. Setting $\mathbf{s}_{\mathbf{a}}^{(t+1)} = \nabla_{\mathbf{a}} F(\mathbf{W}^{(t+1)}) - \boldsymbol{\mu} \in \partial_{\mathbf{a}} F(\mathbf{W}^{(t+1)})$ and applying the same Lipschitz argument yields $\|\mathbf{s}_{\mathbf{a}}^{(t+1)}\| \leq LM^2\|\mathbf{W}^{(t+1)} - \mathbf{W}^{(t)}\|$.

Combining all three components and setting $\lambda = 3LM^2$ completes the proof.
\end{proof}

With these two lemmas in place, we now restate the main convergence theorem and proceed to its full proof.
\begin{supptheorem}[Convergence of \texttt{cBMM}]
Under Assumptions \textnormal{(H1)} and \textnormal{(H2)}, the sequence $\{(\boldsymbol{\theta}^{(t)}, \mathbf{a}^{(t)}, \mathbf{b}^{(t)})\}$ generated by Algorithm \ref{alg:cBMM} satisfies:
\begin{enumerate}
    \item[(i)] \textnormal{(Summability)} $\displaystyle\sum_{t=0}^{\infty}\|\mathbf{W}^{(t+1)} - \mathbf{W}^{(t)}\| < \infty$;
    \item[(ii)] \textnormal{(Whole-sequence convergence)} $\mathbf{W}^{(t)} \to \mathbf{W}^* = (\boldsymbol{\theta}^*, \mathbf{a}^*, \mathbf{b}^*)$ as $t \to \infty$;
    \item[(iii)] \textnormal{(Stationarity)} $\mathbf{W}^*$ satisfies the KKT conditions:
    \[
    \nabla_{\boldsymbol{\theta}}\ell(\mathbf{X}^*) = \mathbf{0}, \quad \mathbf{a}^* \geq \mathbf{0}, \quad \nabla_{\mathbf{a}}\ell(\mathbf{X}^*) \geq \mathbf{0}, \quad \mathbf{a}^* \circ \nabla_{\mathbf{a}}\ell(\mathbf{X}^*) = \mathbf{0}.
    \]
\end{enumerate}
Furthermore, when the KL exponent is $1/2$, the convergence is \textbf{linear}, i.e., there exist constants $C > 0$ and $\alpha \in (0,1)$ such that $\|\mathbf{W}^{(t)} - \mathbf{W}^*\| \leq C\alpha^t$.
\end{supptheorem}

\begin{proof}
\textit{Step 1: Monotone decrease.} By Lemma \ref{lem:descent}, $\{F(\mathbf{W}^{(t)})\}$ is monotonically non-increasing. Since $\ell(\mathbf{X}) \geq 0$ for all $\mathbf{X}$, the sequence converges to a finite limit $F^* \geq 0$.

\textit{Step 2: KL property.} The function $F(\mathbf{W}) = \ell(\boldsymbol{\theta}\mathbf{a}^\top + \mathbf{1}\mathbf{b}^\top)$ is a composition of the logistic loss with a bilinear map. Since the logistic loss is sub-analytic and the bilinear map is polynomial, $F$ is sub-analytic and hence belongs to an o-minimal structure \citep{KL2}. Therefore $F$ satisfies the KL property at every critical point, with KL exponent being $1/2$.

\textit{Step 3: Summability via KL inequality.} By Assumption (H1), there exists a subsequence $\mathbf{W}^{(t_k)} \to \mathbf{W}^*$ with $F(\mathbf{W}^*) = F^*$. By the KL property, there exist $\eta > 0$, a neighborhood $\mathcal{U}$ of $\mathbf{W}^*$, and a concave function $\varphi: [0,\eta) \to \mathbb{R}_+$ with $\varphi(0) = 0$ and $\varphi' > 0$, such that for all $\mathbf{W} \in \mathcal{U}$ with $F^* < F(\mathbf{W}) < F^* + \eta$,
\[
\varphi'(F(\mathbf{W}) - F^*)\|\mathbf{s}\| \geq 1, \quad \forall \mathbf{s} \in \partial F(\mathbf{W}).
\]
For sufficiently large $t$, let $\Delta_t = \varphi(F(\mathbf{W}^{(t)}) - F^*) - \varphi(F(\mathbf{W}^{(t+1)}) - F^*)$. By concavity of $\varphi$, Lemma \ref{lem:descent}, the KL inequality, and Lemma \ref{lem:error},
\[
\Delta_t \geq \varphi'(F(\mathbf{W}^{(t)}) - F^*) \cdot \rho\|\mathbf{W}^{(t+1)} - \mathbf{W}^{(t)}\|^2 \geq \frac{\rho}{\lambda} \cdot \frac{\|\mathbf{W}^{(t+1)} - \mathbf{W}^{(t)}\|^2}{\|\mathbf{W}^{(t)} - \mathbf{W}^{(t-1)}\|}.
\]
Applying the AM-GM inequality $a^2/b \geq 2a - b$ and summing over $t = T, \ldots, T'$ yields a telescoping bound:
\[
\sum_{t=T}^{T'}\|\mathbf{W}^{(t+1)} - \mathbf{W}^{(t)}\| \leq \frac{\lambda}{2\rho}\varphi(F(\mathbf{W}^{(T)}) - F^*) + \frac{1}{2}\|\mathbf{W}^{(T)} - \mathbf{W}^{(T-1)}\| < \infty.
\]
Letting $T' \to \infty$ establishes (i).

\textit{Step 4: Whole-sequence convergence and stationarity.} Part (i) implies $\{\mathbf{W}^{(t)}\}$ is a Cauchy sequence, so $\mathbf{W}^{(t)} \to \mathbf{W}^*$, establishing (ii). By Lemma \ref{lem:error}, $\|\mathbf{s}^{(t+1)}\| \to 0$. Since $\partial F$ has a closed graph, passing to the limit gives $\mathbf{0} \in \partial F(\mathbf{W}^*)$, which is equivalent to the KKT conditions in (iii). The linear convergence rate follows from the KL exponent of $1/2$ \citep{KL3}.
\end{proof}

\subsection{Per-iteration computational complexity}\label{supp:complexity}
The computation bottleneck of Algorithm \ref{alg:cBMM} lies in evaluating a matrix of the form $(\mathbf{uv}_1^{\top})_{\Omega}$, where $(\cdot)_{\Omega}$ is the masking operator that preserves entries $(i, j) \in \Omega$ and sets all others to zero.
Note that the active set method \citep{NNLS} used to solve
\[\arg\min_{\mathbf{v}_1\geq\mathbf{0}}\|\mathbf{u}\mathbf{v}_1^\top - \mathbf{R}\|_{\text{F}(\Omega)}^2\]
in line 7 also requires computing a matrix of the form $(\mathbf{uv}_1^{\top})_{\Omega}$ at each of its iterations.
Naively computing $\mathbf{uv}_1^{\top}$ and then masking out entries in $\Omega^c$ costs $O(NJ)$ flops. However, this complexity can be reduced to $O(|\Omega|)$ by directly evaluating only the observed entries, which is a substantial improvement when $|\Omega| \ll NJ$. We detail this construction below.

Let $\mathbf{A} \equiv (\mathbf{uv}_1^{\top})_{\Omega}$ and let $\Omega_i$ denote the set of indices $j$ such that $(i, j) \in \Omega$. Since we only need to compute entries $j \in \Omega_i$ for the $i$-th row of $\mathbf{A}$, we denote this nonzero subset as $\mathbf{A}(i, \Omega_i)$ and observe that
\[\mathbf{A}(i, \Omega_i) = {u}_i\mathbf{v}_{1,\Omega_i},\]
where $\mathbf{v}_{1,\Omega_i}$ is the sub-vector of $\mathbf{v}_1$ containing only the entries indexed by $\Omega_i$. Consequently, computing the $i$-th row of $\mathbf{A}$ incurs a cost of $O(|\Omega_i|)$ flops. Summing over all $N$ rows yields a total computational complexity of $O(|\Omega|)$ flops per-iteration.

\subsection{Algorithmic details}\label{supp:algdetail}
\textbf{EM algorithm for \texttt{mirt}.}\quad
We implement by setting itemtype as 2PL in the \texttt{mirt} function of \texttt{R} \citep{mirt}. Furthermore,
in the M-step, we specify \texttt{L-BFGS-B} optimizer \citep{LBFGSB} embedded in the \texttt{R} function to ensure non-negativity constraint. Subsequently, Expected A Posteriori (EAP) method is employed to derive the final estimates of $\boldsymbol\theta$ post-convergence.

\textbf{Post-hoc proximal gradient for \texttt{Manopt}.}\quad
The optimization is carried out via conjugate gradient over the product space $(\boldsymbol\theta, \mathbf{a}, \mathbf{b})\in \mathbb{R}^N \times  \mathbb{R}^J_+ \times \mathbb{R}^J$, and incorporates a post-hoc proximal gradient step to enforce the constraint $\mathbf{a}\geq0$. Specifically, if denoting $(\hat{\boldsymbol\theta},\hat{\mathbf{a}},\hat{\mathbf{b}})$ as the optimum from vanilla \texttt{Manopt}, it projects
\[\begin{cases}
\begin{aligned}
\tilde{\mathbf{b}}^{(i+1)} &= \tilde{\mathbf{b}}^{(i)} - \eta \nabla_{\tilde{\mathbf{b}}^{(i)}} \ell(\tilde{\mathbf{X}}^{(i)}) \\
\tilde{\mathbf{a}}^{(i+1)} &= \text{proj}_{\mathbb{R}_+} \left(\tilde{\mathbf{a}}^{(i)} - \eta (\nabla_{\tilde{\mathbf{a}}^{(i)}} \ell(\tilde{\mathbf{X}}^{(i)}) + \lambda) \right)
\end{aligned}
\end{cases}\]
with initial $(\tilde{\mathbf{a}}^{(1)},\tilde{\mathbf{b}}^{(1)}) = (\hat{\mathbf{a}},\hat{\mathbf{b}})$, where $\text{proj}_{\mathbb{R}_+}(x) = \max(0, x)$ is the proximal operator to non-negative orthant. $\eta$, $\lambda$ control the step size and sparsity level. The final output then becomes $(\hat{\boldsymbol\theta},\tilde{\mathbf{a}},\tilde{\mathbf{b}})$.

\textbf{Hyperparameters.}\quad
In \texttt{mirt}, the latent model trait follows a prior distribution of $\theta_i\overset{\text{i.i.d.}}{\sim} \mathrm{N}(0,1)$, with initial values for item parameters  $\log a_j\overset{\text{i.i.d.}}{\sim}  \mathrm{N}(0,1)$ and $b_j\overset{\text{i.i.d.}}{\sim} \mathrm{N}(0,1)$. The same prior for $\boldsymbol\theta$ is carried over during EAP. Similarly, for \texttt{cBMM}, \texttt{L-BFGS-B} and \texttt{Manopt}, the related parameters $\theta_i$, $\log a_j$, and $b_j$ are initialized using the same distributions.
All algorithms are set to a maximum of 1,000 iterations/epochs with a convergence tolerance of $10^{-4}$. For \texttt{py-irt}, we adhere to default configurations as in the package except the epoch limit. Computations are executed on Intel Xeon E5-2683 v4 processors, where each replicate is assigned to one dedicated CPU core with 64 GB of allocated RAM.

\subsection{Simulated studies for sensitivity analysis}\label{supp:sensitivity}
We reexamine the computational scalability of different methods using synthetic examples, by varying the evaluation scale, defined by the number of models $N$ and benchmark items $J$.
Additionally, to demonstrate the robustness of \texttt{cBMM}, we conduct sensitivity analysis by varying several  factors during synthetic data generation, including the temperature $\sigma$, the missing rate within $\mathbf{Y}$, $\rho=1-|\Omega|/NJ$ under different missingness patterns, and the redundancy level in benchmark items, as reflected by the sparsity level of $\mathbf{a}$, defined as $s=1-\|\mathbf{a}\|_0/J$. Here, $a_j = 0$ indicates that the $j$-th item lacks discriminative power for distinguishing model performance and is therefore considered redundant.

\textbf{Data generation process.}\quad
The default data generation process follows 50 replicates for each design. The response $y_{ij}$ is generated under logit link (\ref{equ:logit}) based on the underlying score $x_{ij}$ reparameterized by (\ref{equ:repara}), with $\theta_i\overset{\text{i.i.d.}}{\sim} \mathrm{N}(0,1)$ for $i=1,...,N$, $a_j\overset{\text{i.i.d.}}{\sim} \text{uniform}(0.5,1)$ and $b_j\overset{\text{i.i.d.}}{\sim} \mathrm{N}(0,0.5)$ for $j=1,...,J$, which follows the same data generation procedure in \cite{lart}. Unless otherwise specified, default parameters are set to $\sigma=1, \rho=0, s=0, N=1000,$ and $J=1000$.

\textbf{Performance metrics.}\quad
We evaluate the performance of the compared methods using the following metrics. (i) The Spearman correlation coefficient between the estimate and the ground truth for $\boldsymbol\theta$ and $\mathbf{a}$. (ii) The Root Mean Square Error (RMSE) of the estimate for $\mathbf{b}$. (iii) The relative error of the score matrix, defined as $\|\mathbf{X}-\hat{\mathbf{X}}\|_{\text{F}}^2 / \|\mathbf{X}\|_{\text{F}}^2$. (iv) The Hellinger distance between the estimated and true distributions, $\Phi(\hat{\mathbf{X}})$ and $\Phi(\mathbf{X})$, which is defined as  \citep{hdist2, hdist1}
\[
\frac{1}{NJ}\sum_{i,j}d_\text{H}^2(\Phi(\hat{x}_{ij}),\Phi(x_{ij})),
\]
where $d_\text{H}^2(p,q) = (\sqrt{p}-\sqrt{q})^2 + (\sqrt{1-p}-\sqrt{1-q})^2$. The Hellinger distance is a standard metric between two distributions. It is non-negative and equals zero when the two distributions are identical, which provides a different perspective on estimation accuracy compared to the relative error. (v) Runtime in seconds.

\subsubsection{Evaluation size}
To account for variations in the scale of models and benchmark items across different leaderboards \citep{open-llm-leaderboard,matharena}, we reexamine the computational scalability   of different methods by synthetic examples. We evaluate two scenarios: fixed $N$ and fixed $J$. In both cases, we vary the other dimension over $\{200,500,1000,2000,5000\}$.

Figure \ref{exp:matsizefixN} and \ref{exp:matsizefixJ} present results under fixed $N$ and fixed $J$ settings, respectively. Notably, we observe that \texttt{py-irt} exhibits substantially worse estimation performance than the other methods. This is evidenced by
lower Spearman correlations between the estimated item discrimination parameters $\hat{\mathbf{a}}$ and their ground truth values, higher RMSE for the estimated item easiness parameters $\hat{\mathbf{b}}$, and higher relative errors in the recovered score matrix $\hat{\mathbf{X}}$, along with substantially longer runtimes under the same iteration budget. We therefore include \texttt{py-irt} for completeness only and do not regard it as a primary baseline for comparison.


We also observe algorithmic instability in \texttt{mirt}. When $J\in\{2000,5000\}$, the \texttt{mirt} solver produces `NA' results due to  singularities in response patterns, a common issue in the psychometrics literature  \citep{irtNA,mirt}.
This may occur, for instance, when items exhibit near-constant responses with minimal variance across respondents.
Furthermore, when two items induce nearly identical response patterns, singularities can also arise when implementing \texttt{mirt}.
Other factors, such as extreme temperatures, specific missing patterns, or item redundancy, further exacerbate these numerical instabilities, as demonstrated in Section \ref{supp:temperature}, \ref{supp:missing} and \ref{supp:redun}. Overall, this highlights the limitations of \texttt{mirt} in large-scale and complex settings.

Regarding the recovery of the score matrix $\mathbf{X}$, both evaluation scale settings exhibit a consistent decrease in relative error as the number of models or benchmark questions increases. However, the recovery behaviors of $\boldsymbol{\theta}$, $\mathbf{a}$, and $\mathbf{b}$ differ across the two scenarios. Fixing $N$ while increasing $J$ leads to more accurate recovery of model abilities $\boldsymbol{\theta}$, but generally provides limited improvement in estimating item features $\mathbf{a}$ and $\mathbf{b}$, since only the evaluation coverage of each model is expanded through additional benchmark items. Conversely, fixing $J$ while increasing $N$ improves estimation of the item parameters, whereas the accuracy of model ability estimates remains relatively stable because the underlying benchmark set is unchanged. These findings are consistent with established observations in psychometric IRT literature \citep{irt1,irtNA,irt2}.

Combined with the runtime analysis, \texttt{cBMM} produces results that are highly consistent with empirical expectations while substantially reducing computational overhead, highlighting its strong scalability and practical applicability.



\begin{figure}[t]
 \centering
 \includegraphics[width=\linewidth]{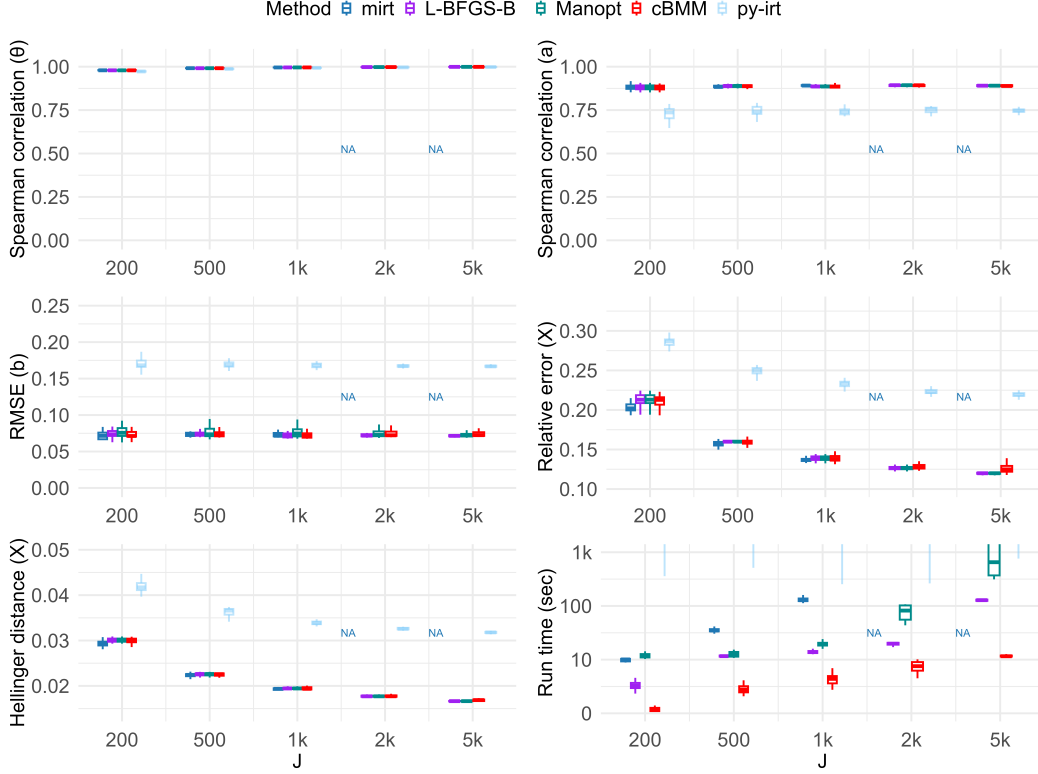}
 \caption{Scalability comparison under varying numbers of items $J$. The number of models is fixed at $N=1000$, while $J$ increases from 200 to 5000. Each boxplot summarizes 50 independent replicates. Runtimes exceeding 1,000 seconds are truncated for better visualization. `NA' indicates convergence failure due to ill-conditioned data configurations in \texttt{mirt}.}\label{exp:matsizefixN}
\end{figure}
\begin{figure}[t]
 \centering
 \includegraphics[width=\linewidth]{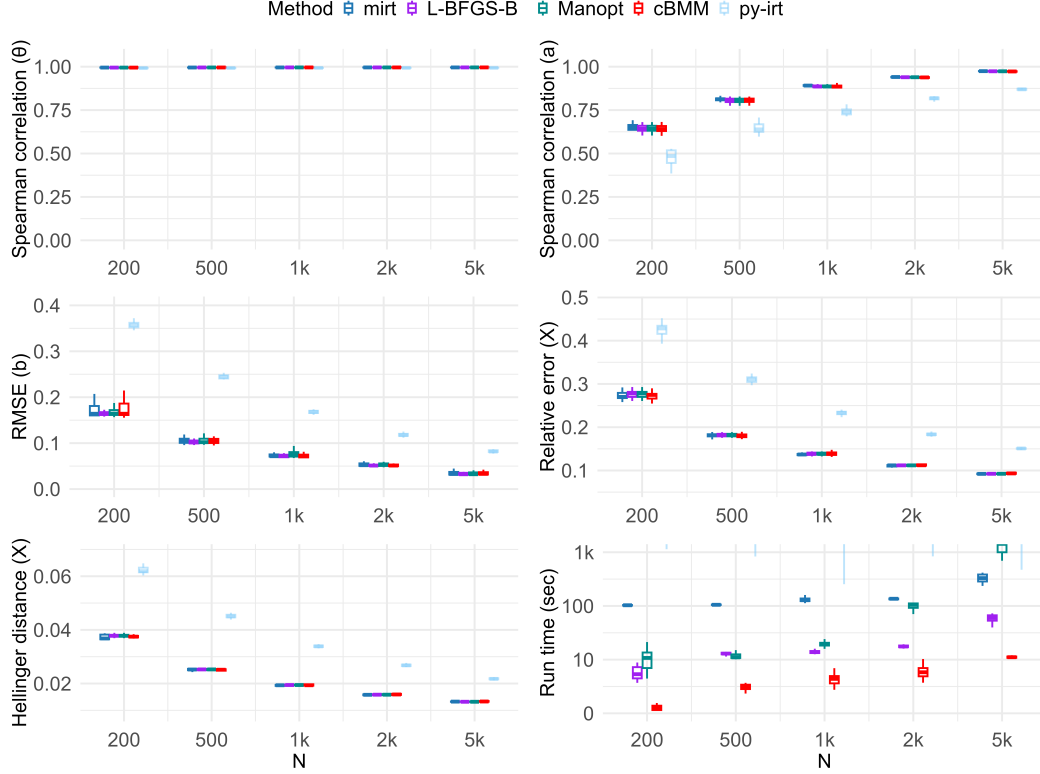}
 \caption{Scalability comparison under varying numbers of LLMs $N$. The number of benchmark items is fixed at $J=1000$, while $N$ increases from 200 to 5000. Each boxplot summarizes 50 independent replicates. Runtimes exceeding 1,000 seconds are truncated for better visualization.}\label{exp:matsizefixJ}
\end{figure}

\subsubsection{Temperature}\label{supp:temperature}
The temperature $\sigma$ acts as a pivotal scaling parameter for signal processing, with demonstrated efficacy in confidence calibration \citep{intro16,temp2} and contrastive multi-modal alignment \citep{temp3,temp4}. We therefore investigate the sensitivity of evaluation outcomes under different values of $\sigma$.

\begin{figure}[t]
 \centering
 \includegraphics[width=\linewidth]{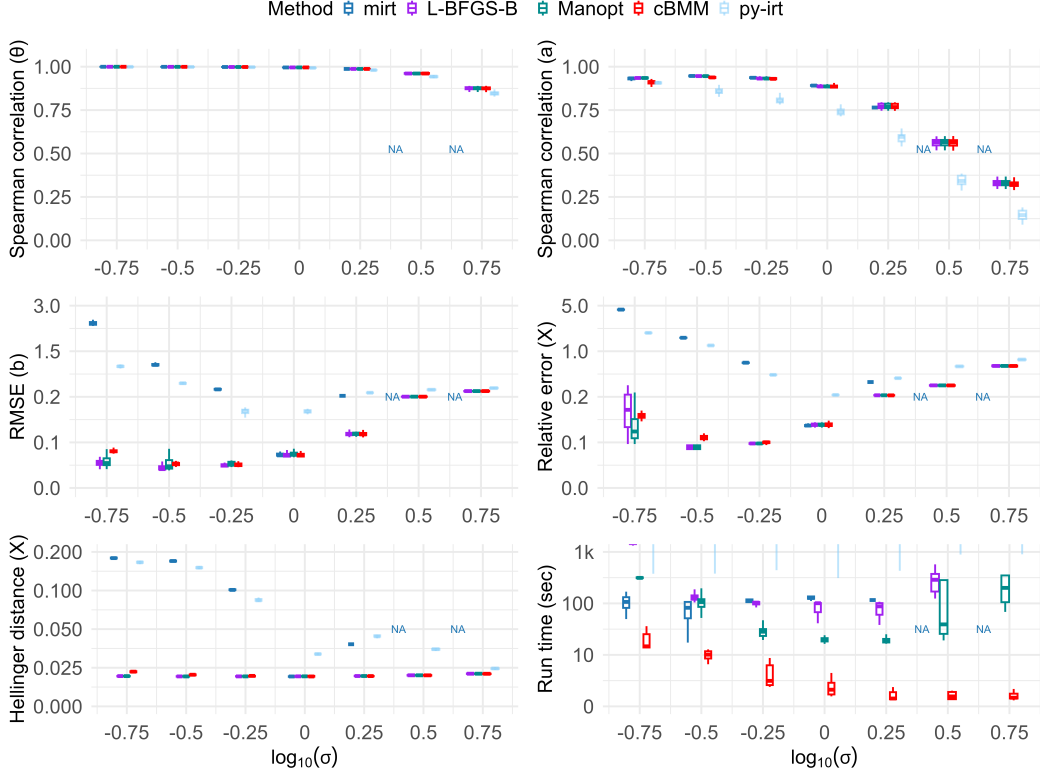}
 \caption{Sensitivity analysis under varying temperature $\sigma$. Each boxplot summarizes 50 independent replicates. Runtimes exceeding 1,000 seconds are truncated for better visualization. `NA' indicates convergence failure due to  ill-conditioned data configurations in \texttt{mirt}.}\label{exp:noise}
\end{figure}

Figure \ref{exp:noise} illustrates performances of different methods  across varying settings of $\sigma$ where $\log_{10}(\sigma)\in\{-0.75,-0.5,-0.25,0,0.25,0.5,0.75\}$.
We again observe algorithmic instability in \texttt{mirt}: when $\log_{10}(\sigma)\in\{0.5,0.75\}$,  it  produces `NA' results due to singularities in response patterns, whereas \texttt{cBMM} remains numerically stable.
Furthermore, \texttt{cBMM} demonstrates enhanced robustness across varying temperature levels. As evidenced in Figure \ref{exp:noise}, in terms of RMSE for $\mathbf{b}$, relative error and Hellinger distance for $\mathbf{X}$, all methods exhibit performance degradation when the temperature is either too low or too high compared with moderate setting, which echoes the discussion in \cite{hdist2} regarding ill-posed recovery problem under extreme temperatures. Nevertheless, \texttt{cBMM} maintains consistently low variability across different values of $\sigma$, highlighting its robustness relative to competing methods. In contrast, the performance of \texttt{mirt} and \texttt{py-irt} follows a pronounced U-shaped pattern, with substantial fluctuations in the recovery of $\mathbf{b}$ and $\mathbf{X}$. Moreover, \texttt{py-irt} suffers from severe degradation in item discrimination estimation as temperature increases. These results further highlight the limitations of these methods in LLM evaluation tasks.

Compared with the two generic solvers  \texttt{L-BFGS-B} and \texttt{Manopt}, \texttt{cBMM} exhibits comparable recovery performance while reducing runtime by one to two orders of magnitude (up to  200x speedup), depending on the temperature. This advantage is consistent across all temperature settings, underscoring the superior computational efficiency of \texttt{cBMM}.

\subsubsection{Missingness}\label{supp:missing}
In this section, we analyze multiple missingness patterns: Missing Completely At Random (MCAR), Missing At Random (MAR) and Missing Not At Random (MNAR), with missing rates ($\rho = 1 - |\Omega|/(NJ)$). These patterns are designed to mimic real-world scenarios such as interfacing latency, generalization failures of LLMs on domain-specific items \citep{missing2}, or benchmark items triggering the safety protocols of participant models due to implicit unsafe flags \citep{missing3}.

\textbf{Missing Completely At Random.}\quad
Interfacing latency, exceeded rate limits, or safety refusals often incur empty responses \citep{missing1,lart}, resulting in missing values in the response matrix $\mathbf{Y}$. Assuming an MCAR pattern that aligns with the external nature of these aforementioned failures, we vary the missing rate $\rho$ from 0 to 0.8.

\begin{figure}[t]
 \centering
 \includegraphics[width=\linewidth]{figures/missing_rate.pdf}
 \caption{Sensitivity analysis under varying missing rate $\rho$ with an MCAR missingness pattern. Each boxplot summarizes 50 independent replicates. Runtimes exceeding 1,000 seconds are truncated for better visualization.}\label{exp:missing}
\end{figure}

As illustrated in Figure \ref{exp:missing}, \texttt{py-irt} consistently underperforms across nearly all recovered parameters, and is included as a baseline only for completeness. The other methods achieve comparable estimation accuracy across all metrics, and, as expected, their performance deteriorates as the missing rate increases.
Notably, \texttt{cBMM} emerges as the most computationally efficient approach across all missing rates, achieving up to a 50x speedup over other methods.

\textbf{Missing At Random.}\quad
We design structural missing patterns based on the spatial positions of rows and columns to reflect potential generalization failures of LLMs on specific items. Specifically, we define a spatial metric $Sp_{ij}$ as the average of the normalized row and column indices:
\[Sp_{ij} = \frac{r_i + c_j}{2}, \quad \text{where}~r_i = \frac{i-1}{N-1},~c_j = \frac{j-1}{J-1}.\]
The probability of entry $(i, j)$ being missing is then determined by the logistic function
\[\Phi(\alpha + \beta \cdot {Sp}_{ij}) = \frac{1}{1 + \exp(-\alpha - \beta \cdot Sp_{ij})}\]
following equation (\ref{equ:logit}), where $\beta=2$ and $\alpha$ is calibrated to achieve a target overall missing rate $\rho$. This construction induces a graded missingness pattern, where the probability of missing entries increases with larger row or column indices.

The results are summarized in Figure \ref{supp:missing1}. We observe a pattern consistent with Figure \ref{exp:missing}, where \texttt{cBMM} achieves comparable ranking recovery performance for $\boldsymbol{\theta}$ and $\mathbf{a}$, as well as comparable estimation accuracy for $\mathbf{b}$ and $\mathbf{X}$, while maintaining superior computational efficiency.

\begin{figure}[t]
 \centering
 \includegraphics[width=\linewidth]{figures/missing_rate2.pdf}
 \caption{Sensitivity analysis under varying missing rate $\rho$ with an MAR missingness pattern. Each boxplot summarizes 50 independent replicates. Runtimes exceeding 1,000 seconds are truncated for better visualization.}\label{supp:missing1}
\end{figure}

\textbf{Missing Not At Random.}\quad
To simulate a MNAR pattern, we implement a masking mechanism for -1 entries (incorrect responses) in $\mathbf{Y}$. Specifically, we introduce a gate variable $G_j$ for each column $j$, where
$G_j \overset{\text{i.i.d.}}{\sim} \text{Bernoulli}(\rho_c)$,
and $\rho_c$ represents the column-wise missing rate. If $G_j = 0$, no entries in column $j$ are missing, while if $G_j = 1$, all -1 entries within that column are masked.
This setting mimics scenarios where LLMs fail to generalize on domain-specific items or where benchmark items trigger safety mechanisms in participant models.

The results are summarized in Figure \ref{supp:missing2}. Notably, \texttt{mirt} consistently returns `NA', even under less severe missingness patterns, due to numerical singularities induced by complex missing structures. In contrast, \texttt{cBMM} maintains high rank coherence of the recovered parameters and high estimation accuracy of latent scores, while achieving the shortest runtime among all methods.
\begin{figure}[t]
 \centering
 \includegraphics[width=\linewidth]{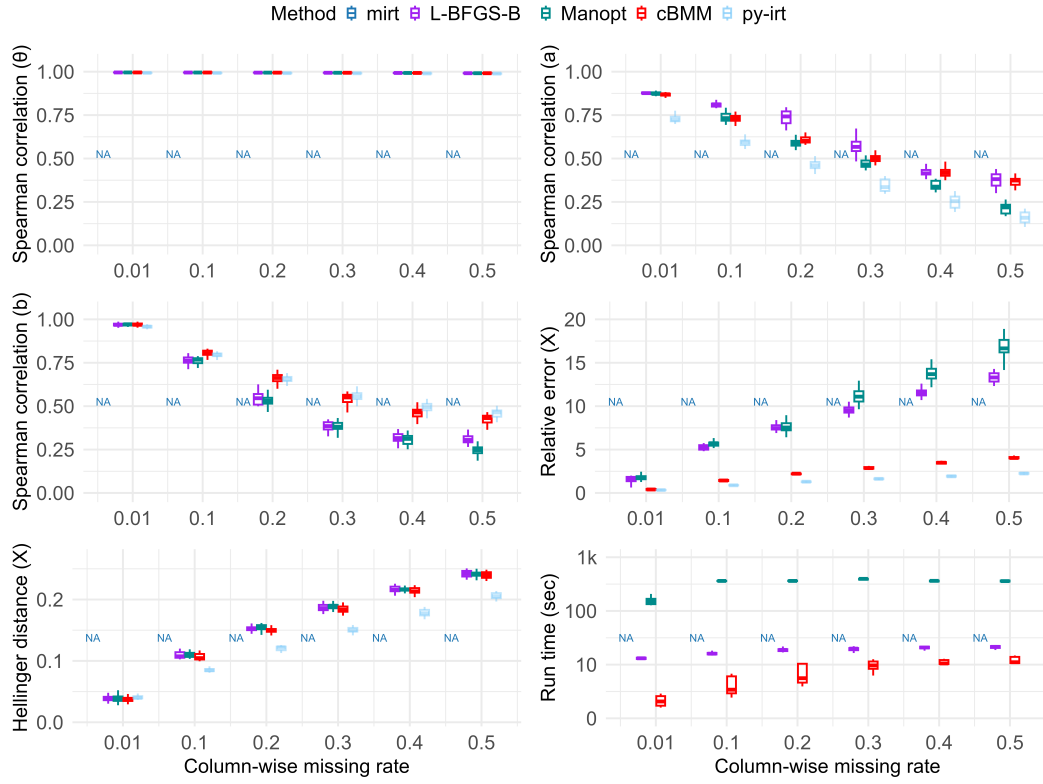}
 \caption{Sensitivity analysis for missing rate $\rho$ under MNAR pattern. Each boxplot summarizes 50 independent replicates. Runtimes exceeding 1,000 seconds are truncated for better visualization. `NA' indicates convergence failure due to ill-conditioned data configurations in \texttt{mirt}.}\label{supp:missing2}
\end{figure}

\subsubsection{Benchmark redundancy}\label{supp:redun}
As benchmarks evolve, we expect to see increasing item redundancy across participant models \citep{redun1}, which may stem from information leakage in training data over time \citep{redun2}. Such benchmark items provide limited discriminative power and should be phased out in future evaluations. A representative example is the evolution from SWE-bench \citep{swebench} to SWE-bench Verified \citep{swebenchverified}, where a more discriminative subset of benchmark items was identified through extensive human annotation and verification. In contrast, we approach this problem from a data-driven perspective by identifying sparse discrimination parameters $\mathbf{a}$, thereby reducing reliance on manual curation and human labor.

We use the sparsity level of $\mathbf{a}$ to reflect the redundancy level in the benchmark and consider $s = 1 - \|\mathbf{a}\|_0/J \in \{0.05, 0.1, 0.2, 0.3, 0.4, 0.5\}$ to validate method performance under varying levels of item redundancy, where the zero mask is assigned completely at random.

The results are summarized in Figure \ref{exp:sparsity}. \texttt{py-irt} again exhibits significantly inferior recovery performance across all metrics. \texttt{mirt} reports `NA' at $s \in \{0.4, 0.5\}$, as the singularity issue becomes more severe when benchmark items are increasingly redundant. In contrast, \texttt{cBMM} demonstrates consistently accurate parameter recovery while maintaining superior computational efficiency.

We further dissect the result by focusing  on zero and non-zero patterns. In line with traditional hypothesis testing frameworks, we interpret the recovery of non-zero parameters as rejecting the null signal. Intuitively, $a_j = 0$ indicates that the $j$-th item has no discriminative power for distinguishing LLM abilities and should be identified and removed in subsequent iterations.
We additionally consider two metrics to evaluate the accuracy of this discovery:
\[\text{Recall (a)}=\frac{TP}{TP+FN},\quad \text{Precision (a)}= \frac{TP}{TP+FP},\]
where $TP=|\{j:\hat{a}_j\neq0,a_j\neq0\}|$, $FN=|\{j:\hat{a}_j=0,a_j\neq0\}|$ and $FP=|\{j:\hat{a}_j\neq0,a_j=0\}|$.
Interestingly, all methods tend to over-reject the null: while this leads to a recall of 1, it fails to control the false discovery rate. This issue is particularly pronounced for \texttt{py-irt}, which fails to provide meaningful sparse recovery or redundancy identification. These results suggest that developing principled error-control mechanisms is an important direction for future work.

\begin{figure}[t]
 \centering
 \includegraphics[width=\linewidth]{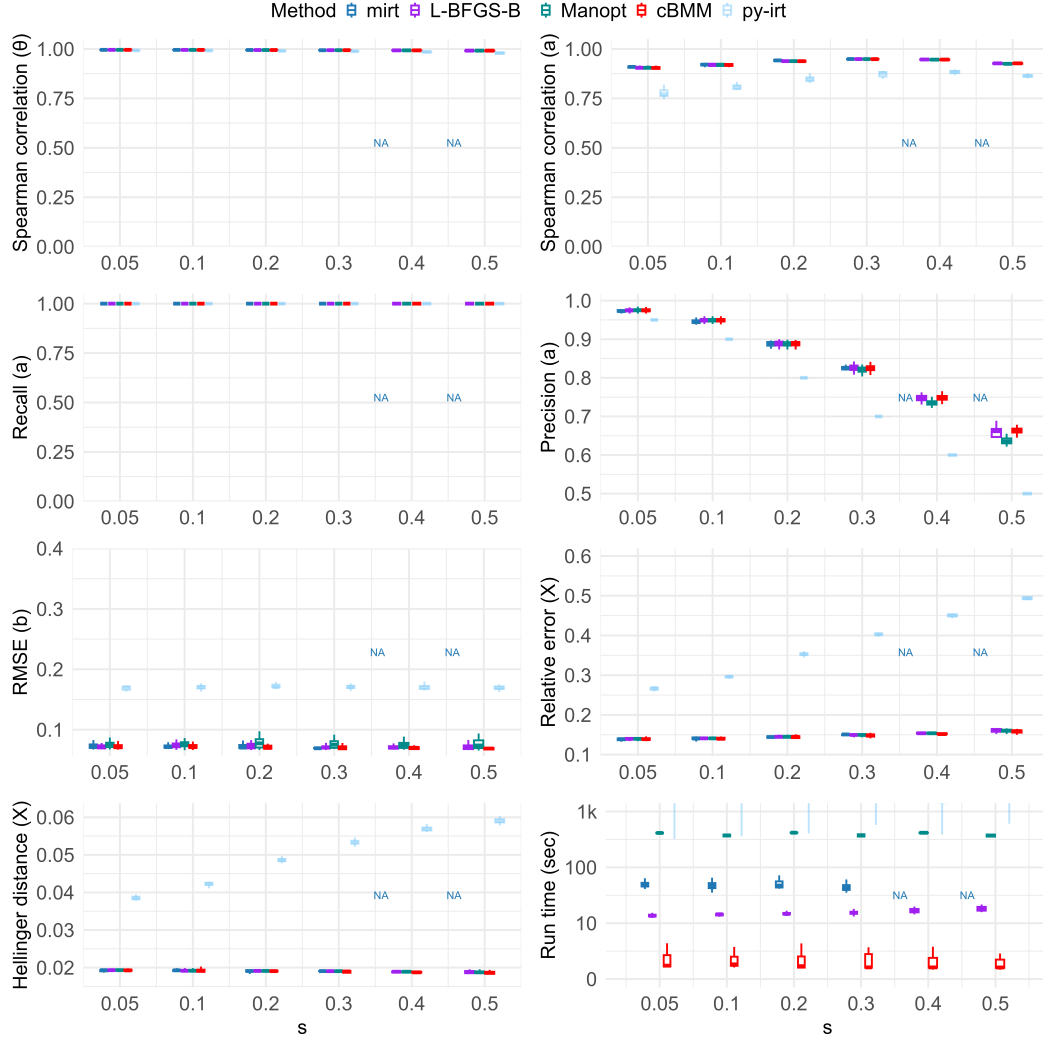}
 \caption{Sparsity recovery analysis for benchmark redundancy detection. Each boxplot summarizes 50 independent replicates. Runtimes exceeding 1,000 seconds are truncated for better visualization. `NA' indicates convergence failure due to ill-conditioned data configurations in \texttt{mirt}.}\label{exp:sparsity}
\end{figure}

\subsection{Recovery stability}
To validate the robustness of \texttt{cBMM} to initialization, we conduct experiments under various random initialization strategies following \cite{irt3} on six real-world benchmark suites from the Open LLM Leaderboard. We run 50 replicates for each configuration to analyze the inter- and intra-run Spearman correlations. Specifically, for all $i \in \{1, \dots, N\}$ and $j \in \{1, \dots, J\}$, we initialize $\theta_i \overset{\text{i.i.d.}}{\sim} \text{N}(0,1)$, $\log a_j \overset{\text{i.i.d.}}{\sim} \text{N}(0,\sigma^2_{a})$, and $b_j \overset{\text{i.i.d.}}{\sim} \text{N}(0,\sigma^2_{a})$, where $\sigma^2_{a} \in \{1, 2, 3, 4, 5\}$. We find that the recovered model and item parameters exhibit consistently high correlations across different seeds and initialization settings, as shown in Table \ref{tab:stable1}, \ref{tab:stable2}, \ref{tab:stable3}, suggesting that \texttt{cBMM} yields stable and reliable parameter recovery.

\begin{table}[t]
\centering
\caption{Stability check for \texttt{cBMM}-recovered $\hat{\boldsymbol\theta}$. Each block corresponds to a specific benchmark suite. Cells report the mean and standard deviation (in parenthesis) of Spearman correlation coefficients across 50 replicates.}
\label{tab:stable1}
\tiny
\begin{tabular}{@{}ccccccccccc@{}}
\toprule
$\hat{\boldsymbol\theta}$   & \multicolumn{5}{c}{IFEval}                                                                                                                                                                                                                                                                                                    & \multicolumn{5}{c}{BBH}                                                                                                                                                                                                                                                                                                     \\ \midrule
$\sigma^2_{a}$ & 1                                                             & 2                                                             & 3                                                              & 4                                                              & 5                                                           & 1                                                             & 2                                                             & 3                                                             & 4                                                             & 5                                                           \\
1   & \begin{tabular}[c]{@{}c@{}}1\\      (2.55e-17)\end{tabular}   &                                                               &                                                                &                                                                &                                                             & \begin{tabular}[c]{@{}c@{}}1\\      (5.70e-17)\end{tabular}   &                                                               &                                                               &                                                               &                                                             \\
2   & \begin{tabular}[c]{@{}c@{}}0.99\\      (9.58e-7)\end{tabular} & 1                                                             &                                                                &                                                                &                                                             & \begin{tabular}[c]{@{}c@{}}0.99\\      (1.35e-7)\end{tabular} & \begin{tabular}[c]{@{}c@{}}1\\      (5.09e-17)\end{tabular}   &                                                               &                                                               &                                                             \\
3   & \begin{tabular}[c]{@{}c@{}}0.99\\      (2.07e-6)\end{tabular} & 1                                                             & 1                                                              &                                                                &                                                             & \begin{tabular}[c]{@{}c@{}}0.99\\      (1.42e-7)\end{tabular} & \begin{tabular}[c]{@{}c@{}}0.99\\      (3.02e-8)\end{tabular} & \begin{tabular}[c]{@{}c@{}}1\\      (5.09e-17)\end{tabular}   &                                                               &                                                             \\
4   & \begin{tabular}[c]{@{}c@{}}0.99\\      (2.50e-6)\end{tabular} & 1                                                             & 1                                                              & 1                                                              &                                                             & \begin{tabular}[c]{@{}c@{}}0.99\\      (1.47e-7)\end{tabular} & \begin{tabular}[c]{@{}c@{}}0.99\\      (4.47e-8)\end{tabular} & \begin{tabular}[c]{@{}c@{}}0.99\\      (1.71e-8)\end{tabular} & \begin{tabular}[c]{@{}c@{}}1\\      (5.70e-17)\end{tabular}   &                                                             \\
5   & \begin{tabular}[c]{@{}c@{}}0.99\\      (2.65e-6)\end{tabular} & 1                                                             & 1                                                              & 1                                                              & 1                                                           & \begin{tabular}[c]{@{}c@{}}0.99\\      (1.49e-7)\end{tabular} & \begin{tabular}[c]{@{}c@{}}0.99\\      (5.84e-8)\end{tabular} & \begin{tabular}[c]{@{}c@{}}0.99\\      (3.30e-8)\end{tabular} & \begin{tabular}[c]{@{}c@{}}0.99\\      (1.52e-8)\end{tabular} & \begin{tabular}[c]{@{}c@{}}1\\      (7.20e-17)\end{tabular} \\\midrule
    & \multicolumn{5}{c}{MATH}                                                                                                                                                                                                                                                                                                      & \multicolumn{5}{c}{GPQA}                                                                                                                                                                                                                                                                                                    \\\midrule
$\sigma^2_{a}$ & 1                                                             & 2                                                             & 3                                                              & 4                                                              & 5                                                           & 1                                                             & 2                                                             & 3                                                             & 4                                                             & 5                                                           \\
1   & \begin{tabular}[c]{@{}c@{}}1\\      (0.00e+0)\end{tabular}    &                                                               &                                                                &                                                                &                                                             & \begin{tabular}[c]{@{}c@{}}1\\      (2.55e-17)\end{tabular}   &                                                               &                                                               &                                                               &                                                             \\
2   & \begin{tabular}[c]{@{}c@{}}0.99\\      (1.75e-8)\end{tabular} & \begin{tabular}[c]{@{}c@{}}1\\      (0.00e+0)\end{tabular}    &                                                                &                                                                &                                                             & \begin{tabular}[c]{@{}c@{}}0.99\\      (1.90e-6)\end{tabular} & \begin{tabular}[c]{@{}c@{}}1\\      (3.60e-17)\end{tabular}   &                                                               &                                                               &                                                             \\
3   & \begin{tabular}[c]{@{}c@{}}0.99\\      (1.68e-8)\end{tabular} & \begin{tabular}[c]{@{}c@{}}0.99\\      (1.91e-9)\end{tabular} & \begin{tabular}[c]{@{}c@{}}1\\      (2.55e-17)\end{tabular}    &                                                                &                                                             & \begin{tabular}[c]{@{}c@{}}0.99\\      (2.56e-6)\end{tabular} & \begin{tabular}[c]{@{}c@{}}0.99\\      (4.86e-7)\end{tabular} & \begin{tabular}[c]{@{}c@{}}1\\      (3.60e-17)\end{tabular}   &                                                               &                                                             \\
4   & \begin{tabular}[c]{@{}c@{}}0.99\\      (1.67e-8)\end{tabular} & \begin{tabular}[c]{@{}c@{}}0.99\\      (2.57e-9)\end{tabular} & \begin{tabular}[c]{@{}c@{}}0.99\\      (9.11e-10)\end{tabular} & \begin{tabular}[c]{@{}c@{}}1\\      (2.55e-17)\end{tabular}    &                                                             & \begin{tabular}[c]{@{}c@{}}0.99\\      (2.83e-6)\end{tabular} & \begin{tabular}[c]{@{}c@{}}0.99\\      (7.55e-7)\end{tabular} & \begin{tabular}[c]{@{}c@{}}0.99\\      (9.82e-7)\end{tabular} & \begin{tabular}[c]{@{}c@{}}1\\      (4.41e-17)\end{tabular}   &                                                             \\
5   & \begin{tabular}[c]{@{}c@{}}0.99\\      (1.64e-8)\end{tabular} & \begin{tabular}[c]{@{}c@{}}0.99\\      (3.12e-9)\end{tabular} & \begin{tabular}[c]{@{}c@{}}0.99\\      (1.41e-9)\end{tabular}  & \begin{tabular}[c]{@{}c@{}}0.99\\      (8.90e-10)\end{tabular} & \begin{tabular}[c]{@{}c@{}}1\\      (2.55e-17)\end{tabular} & \begin{tabular}[c]{@{}c@{}}0.99\\      (2.88e-6)\end{tabular} & \begin{tabular}[c]{@{}c@{}}0.99\\      (8.36e-7)\end{tabular} & \begin{tabular}[c]{@{}c@{}}0.99\\      (8.06e-7)\end{tabular} & \begin{tabular}[c]{@{}c@{}}0.99\\      (2.44e-7)\end{tabular} & \begin{tabular}[c]{@{}c@{}}1\\      (4.41e-17)\end{tabular} \\\midrule
    & \multicolumn{5}{c}{MuSR}                                                                                                                                                                                                                                                                                                      & \multicolumn{5}{c}{MMLU-Pro}                                                                                                                                                                                                                                                                                                \\\midrule
$\sigma^2_{a}$ & 1                                                             & 2                                                             & 3                                                              & 4                                                              & 5                                                           & 1                                                             & 2                                                             & 3                                                             & 4                                                             & 5                                                           \\
1   & \begin{tabular}[c]{@{}c@{}}1\\      (7.20e-17)\end{tabular}   &                                                               &                                                                &                                                                &                                                             & \begin{tabular}[c]{@{}c@{}}1\\      (0.00e+0)\end{tabular}    &                                                               &                                                               &                                                               &                                                             \\
2   & \begin{tabular}[c]{@{}c@{}}0.99\\      (2.27e-5)\end{tabular} & \begin{tabular}[c]{@{}c@{}}1\\      (6.24e-17)\end{tabular}   &                                                                &                                                                &                                                             & \begin{tabular}[c]{@{}c@{}}0.99\\      (9.95e-8)\end{tabular} & \begin{tabular}[c]{@{}c@{}}1\\      (7.85e-17)\end{tabular}   &                                                               &                                                               &                                                             \\
3   & \begin{tabular}[c]{@{}c@{}}0.92\\      (3.58e-1)\end{tabular} & \begin{tabular}[c]{@{}c@{}}0.92\\      (3.58e-1)\end{tabular} & \begin{tabular}[c]{@{}c@{}}1\\      (6.24e-17)\end{tabular}    &                                                                &                                                             & \begin{tabular}[c]{@{}c@{}}0.99\\      (1.03e-7)\end{tabular} & \begin{tabular}[c]{@{}c@{}}0.99\\      (1.26e-8)\end{tabular} & \begin{tabular}[c]{@{}c@{}}1\\      (7.85e-17)\end{tabular}   &                                                               &                                                             \\
4   & \begin{tabular}[c]{@{}c@{}}0.92\\      (3.59e-1)\end{tabular} & \begin{tabular}[c]{@{}c@{}}0.92\\      (3.58e-1)\end{tabular} & \begin{tabular}[c]{@{}c@{}}0.99\\      (2.86e-6)\end{tabular}  & \begin{tabular}[c]{@{}c@{}}1\\      (6.24e-17)\end{tabular}    &                                                             & \begin{tabular}[c]{@{}c@{}}0.99\\      (1.10e-7)\end{tabular} & \begin{tabular}[c]{@{}c@{}}0.99\\      (1.52e-8)\end{tabular} & \begin{tabular}[c]{@{}c@{}}0.99\\      (6.28e-9)\end{tabular} & \begin{tabular}[c]{@{}c@{}}1\\      (7.85e-17)\end{tabular}   &                                                             \\
5   & \begin{tabular}[c]{@{}c@{}}0.92\\      (3.59e-1)\end{tabular} & \begin{tabular}[c]{@{}c@{}}0.92\\      (3.58e-1)\end{tabular} & \begin{tabular}[c]{@{}c@{}}0.99\\      (3.47e-6)\end{tabular}  & \begin{tabular}[c]{@{}c@{}}0.99\\      (2.42e-7)\end{tabular}  & \begin{tabular}[c]{@{}c@{}}1\\      (7.20e-17)\end{tabular} & \begin{tabular}[c]{@{}c@{}}0.99\\      (1.11e-7)\end{tabular} & \begin{tabular}[c]{@{}c@{}}0.99\\      (1.77e-8)\end{tabular} & \begin{tabular}[c]{@{}c@{}}0.99\\      (9.97e-9)\end{tabular} & \begin{tabular}[c]{@{}c@{}}0.99\\      (3.66e-9)\end{tabular} & \begin{tabular}[c]{@{}c@{}}1\\      (7.85e-17)\end{tabular} \\ \bottomrule
\end{tabular}
\end{table}

\begin{table}[t]
\centering
\caption{Stability check for \texttt{cBMM}-recovered $\hat{\mathbf{a}}$. Each block corresponds to a specific benchmark suite. Cells report the mean and standard deviation (in parenthesis) of Spearman correlation coefficients across 50 replicates.}
\label{tab:stable2}
\tiny
\begin{tabular}{@{}ccccccccccc@{}}
\toprule
$\hat{\mathbf{a}}$   & \multicolumn{5}{c}{IFEval}                                                                                                                                                                                                                                                                                                  & \multicolumn{5}{c}{BBH}                                                                                                                                                                                                                                                                                                     \\ \midrule
$\sigma^2_{a}$ & 1                                                             & 2                                                             & 3                                                             & 4                                                             & 5                                                           & 1                                                             & 2                                                             & 3                                                             & 4                                                             & 5                                                           \\
1   & \begin{tabular}[c]{@{}c@{}}1\\      (2.55e-17)\end{tabular}   &                                                               &                                                               &                                                               &                                                             & \begin{tabular}[c]{@{}c@{}}1\\      (1.14e-16)\end{tabular}   &                                                               &                                                               &                                                               &                                                             \\
2   & \begin{tabular}[c]{@{}c@{}}0.99\\      (9.58e-7)\end{tabular} & 1                                                             &                                                               &                                                               &                                                             & \begin{tabular}[c]{@{}c@{}}0.99\\      (2.43e-5)\end{tabular} & \begin{tabular}[c]{@{}c@{}}1\\      (1.14e-16)\end{tabular}   &                                                               &                                                               &                                                             \\
3   & \begin{tabular}[c]{@{}c@{}}0.99\\      (2.07e-6)\end{tabular} & 1                                                             & 1                                                             &                                                               &                                                             & \begin{tabular}[c]{@{}c@{}}0.99\\      (2.85e-5)\end{tabular} & \begin{tabular}[c]{@{}c@{}}0.99\\      (6.16e-7)\end{tabular} & \begin{tabular}[c]{@{}c@{}}1\\      (1.14e-16)\end{tabular}   &                                                               &                                                             \\
4   & \begin{tabular}[c]{@{}c@{}}0.99\\      (2.50e-6)\end{tabular} & 1                                                             & 1                                                             & 1                                                             &                                                             & \begin{tabular}[c]{@{}c@{}}0.99\\      (2.61e-5)\end{tabular} & \begin{tabular}[c]{@{}c@{}}0.99\\      (1.15e-6)\end{tabular} & \begin{tabular}[c]{@{}c@{}}0.99\\      (3.18e-7)\end{tabular} & \begin{tabular}[c]{@{}c@{}}1\\      (1.14e-16)\end{tabular}   &                                                             \\
5   & \begin{tabular}[c]{@{}c@{}}0.99\\      (2.65e-6)\end{tabular} & 1                                                             & 1                                                             & 1                                                             & 1                                                           & \begin{tabular}[c]{@{}c@{}}0.99\\      (2.18e-5)\end{tabular} & \begin{tabular}[c]{@{}c@{}}0.99\\      (2.28e-6)\end{tabular} & \begin{tabular}[c]{@{}c@{}}0.99\\      (1.30e-6)\end{tabular} & \begin{tabular}[c]{@{}c@{}}0.99\\      (3.83e-7)\end{tabular} & \begin{tabular}[c]{@{}c@{}}1\\      (1.14e-16)\end{tabular} \\\midrule
    & \multicolumn{5}{c}{MATH}                                                                                                                                                                                                                                                                                                    & \multicolumn{5}{c}{GPQA}                                                                                                                                                                                                                                                                                                    \\\midrule
$\sigma^2_{a}$ & 1                                                             & 2                                                             & 3                                                             & 4                                                             & 5                                                           & 1                                                             & 2                                                             & 3                                                             & 4                                                             & 5                                                           \\
1   & \begin{tabular}[c]{@{}c@{}}1\\      (1.14e-16)\end{tabular}   &                                                               &                                                               &                                                               &                                                             & \begin{tabular}[c]{@{}c@{}}1\\      (0.00e+0)\end{tabular}    &                                                               &                                                               &                                                               &                                                             \\
2   & \begin{tabular}[c]{@{}c@{}}0.99\\      (3.63e-5)\end{tabular} & \begin{tabular}[c]{@{}c@{}}1\\      (1.14e-16)\end{tabular}   &                                                               &                                                               &                                                             & \begin{tabular}[c]{@{}c@{}}0.99\\      (9.34e-5)\end{tabular} & \begin{tabular}[c]{@{}c@{}}1\\      (0.00e+0)\end{tabular}    &                                                               &                                                               &                                                             \\
3   & \begin{tabular}[c]{@{}c@{}}0.99\\      (4.70e-5)\end{tabular} & \begin{tabular}[c]{@{}c@{}}0.99\\      (8.51e-6)\end{tabular} & \begin{tabular}[c]{@{}c@{}}1\\      (1.14e-16)\end{tabular}   &                                                               &                                                             & \begin{tabular}[c]{@{}c@{}}0.99\\      (7.47e-5)\end{tabular} & \begin{tabular}[c]{@{}c@{}}0.99\\      (8.78e-5)\end{tabular} & \begin{tabular}[c]{@{}c@{}}1\\      (0.00e+0)\end{tabular}    &                                                               &                                                             \\
4   & \begin{tabular}[c]{@{}c@{}}0.99\\      (4.95e-5)\end{tabular} & \begin{tabular}[c]{@{}c@{}}0.99\\      (1.04e-5)\end{tabular} & \begin{tabular}[c]{@{}c@{}}0.99\\      (1.13e-6)\end{tabular} & \begin{tabular}[c]{@{}c@{}}1\\      (1.14e-16)\end{tabular}   &                                                             & \begin{tabular}[c]{@{}c@{}}0.99\\      (1.17e-4)\end{tabular} & \begin{tabular}[c]{@{}c@{}}0.99\\      (1.26e-4)\end{tabular} & \begin{tabular}[c]{@{}c@{}}0.99\\      (9.80e-5)\end{tabular} & \begin{tabular}[c]{@{}c@{}}1\\      (0.00e+0)\end{tabular}    &                                                             \\
5   & \begin{tabular}[c]{@{}c@{}}0.99\\      (5.09e-5)\end{tabular} & \begin{tabular}[c]{@{}c@{}}0.99\\      (1.08e-5)\end{tabular} & \begin{tabular}[c]{@{}c@{}}0.99\\      (2.83e-6)\end{tabular} & \begin{tabular}[c]{@{}c@{}}0.99\\      (8.30e-7)\end{tabular} & \begin{tabular}[c]{@{}c@{}}1\\      (1.14e-16)\end{tabular} & \begin{tabular}[c]{@{}c@{}}0.99\\      (1.18e-4)\end{tabular} & \begin{tabular}[c]{@{}c@{}}0.99\\      (1.16e-4)\end{tabular} & \begin{tabular}[c]{@{}c@{}}0.99\\      (1.01e-4)\end{tabular} & \begin{tabular}[c]{@{}c@{}}0.99\\      (3.15e-5)\end{tabular} & \begin{tabular}[c]{@{}c@{}}1\\      (0.00e+0)\end{tabular}  \\\midrule
    & \multicolumn{5}{c}{MuSR}                                                                                                                                                                                                                                                                                                    & \multicolumn{5}{c}{MMLU-Pro}                                                                                                                                                                                                                                                                                                \\\midrule
$\sigma^2_{a}$ & 1                                                             & 2                                                             & 3                                                             & 4                                                             & 5                                                           & 1                                                             & 2                                                             & 3                                                             & 4                                                             & 5                                                           \\
1   & \begin{tabular}[c]{@{}c@{}}1\\      (5.70e-17)\end{tabular}   &                                                               &                                                               &                                                               &                                                             & \begin{tabular}[c]{@{}c@{}}1\\      (7.85e-17)\end{tabular}   &                                                               &                                                               &                                                               &                                                             \\
2   & \begin{tabular}[c]{@{}c@{}}0.99\\      (8.42e-4)\end{tabular} & \begin{tabular}[c]{@{}c@{}}1\\      (7.64e-17)\end{tabular}   &                                                               &                                                               &                                                             & \begin{tabular}[c]{@{}c@{}}0.99\\      (1.44e-5)\end{tabular} & \begin{tabular}[c]{@{}c@{}}1\\      (5.55e-17)\end{tabular}   &                                                               &                                                               &                                                             \\
3   & \begin{tabular}[c]{@{}c@{}}0.92\\      (3.86e-1)\end{tabular} & \begin{tabular}[c]{@{}c@{}}0.92\\      (3.86e-1)\end{tabular} & \begin{tabular}[c]{@{}c@{}}1\\      (6.74e-17)\end{tabular}   &                                                               &                                                             & \begin{tabular}[c]{@{}c@{}}0.99\\      (1.73e-5)\end{tabular} & \begin{tabular}[c]{@{}c@{}}0.99\\      (3.19e-6)\end{tabular} & \begin{tabular}[c]{@{}c@{}}1\\      (5.55e-17)\end{tabular}   &                                                               &                                                             \\
4   & \begin{tabular}[c]{@{}c@{}}0.92\\      (3.87e-1)\end{tabular} & \begin{tabular}[c]{@{}c@{}}0.92\\      (3.87e-1)\end{tabular} & \begin{tabular}[c]{@{}c@{}}0.99\\      (2.26e-4)\end{tabular} & \begin{tabular}[c]{@{}c@{}}1\\      (7.20e-17)\end{tabular}   &                                                             & \begin{tabular}[c]{@{}c@{}}0.99\\      (1.67e-5)\end{tabular} & \begin{tabular}[c]{@{}c@{}}0.99\\      (1.40e-5)\end{tabular} & \begin{tabular}[c]{@{}c@{}}0.99\\      (3.76e-6)\end{tabular} & \begin{tabular}[c]{@{}c@{}}1\\      (0.00e+0)\end{tabular}    &                                                             \\
5   & \begin{tabular}[c]{@{}c@{}}0.92\\      (3.87e-1)\end{tabular} & \begin{tabular}[c]{@{}c@{}}0.92\\      (3.87e-1)\end{tabular} & \begin{tabular}[c]{@{}c@{}}0.99\\      (3.20e-4)\end{tabular} & \begin{tabular}[c]{@{}c@{}}0.99\\      (1.65e-4)\end{tabular} & \begin{tabular}[c]{@{}c@{}}1\\      (7.20e-17)\end{tabular} & \begin{tabular}[c]{@{}c@{}}0.99\\      (2.00e-5)\end{tabular} & \begin{tabular}[c]{@{}c@{}}0.99\\      (2.57e-5)\end{tabular} & \begin{tabular}[c]{@{}c@{}}0.99\\      (1.07e-5)\end{tabular} & \begin{tabular}[c]{@{}c@{}}0.99\\      (1.76e-6)\end{tabular} & \begin{tabular}[c]{@{}c@{}}1\\      (0.00e+0)\end{tabular}  \\ \bottomrule
\end{tabular}
\end{table}

\begin{table}[t]
\centering
\caption{Stability check for \texttt{cBMM}-recovered $\hat{\mathbf{b}}$. Each block corresponds to a specific benchmark suite. Cells report the mean and standard deviation (in parenthesis) of Spearman correlation coefficients across 50 replicates.}
\label{tab:stable3}
\tiny
\begin{tabular}{@{}ccccccccccc@{}}
\toprule
$\hat{\mathbf{b}}$   & \multicolumn{5}{c}{IFEval}                                                                                                                                                                                                                                                                                                  & \multicolumn{5}{c}{BBH}                                                                                                                                                                                                                                                                                                     \\ \midrule
$\sigma^2_{a}$ & 1                                                             & 2                                                             & 3                                                             & 4                                                             & 5                                                           & 1                                                             & 2                                                             & 3                                                             & 4                                                             & 5                                                           \\
1   & \begin{tabular}[c]{@{}c@{}}1\\      (2.55e-17)\end{tabular}   &                                                               &                                                               &                                                               &                                                             & \begin{tabular}[c]{@{}c@{}}1\\      (0.00e+0)\end{tabular}    &                                                               &                                                               &                                                               &                                                             \\
2   & \begin{tabular}[c]{@{}c@{}}0.99\\      (9.58e-7)\end{tabular} & 1                                                             &                                                               &                                                               &                                                             & \begin{tabular}[c]{@{}c@{}}0.99\\      (8.77e-7)\end{tabular} & \begin{tabular}[c]{@{}c@{}}1\\      (0.00e+0)\end{tabular}    &                                                               &                                                               &                                                             \\
3   & \begin{tabular}[c]{@{}c@{}}0.99\\      (2.07e-6)\end{tabular} & 1                                                             & 1                                                             &                                                               &                                                             & \begin{tabular}[c]{@{}c@{}}0.99\\      (9.51e-7)\end{tabular} & \begin{tabular}[c]{@{}c@{}}0.99\\      (1.00e-7)\end{tabular} & \begin{tabular}[c]{@{}c@{}}1\\      (8.82e-17)\end{tabular}   &                                                               &                                                             \\
4   & \begin{tabular}[c]{@{}c@{}}0.99\\      (2.50e-6)\end{tabular} & 1                                                             & 1                                                             & 1                                                             &                                                             & \begin{tabular}[c]{@{}c@{}}0.99\\      (1.03e-6)\end{tabular} & \begin{tabular}[c]{@{}c@{}}0.99\\      (3.07e-7)\end{tabular} & \begin{tabular}[c]{@{}c@{}}0.99\\      (6.15e-8)\end{tabular} & \begin{tabular}[c]{@{}c@{}}1\\      (1.02e-16)\end{tabular}   &                                                             \\
5   & \begin{tabular}[c]{@{}c@{}}0.99\\      (2.65e-6)\end{tabular} & 1                                                             & 1                                                             & 1                                                             & 1                                                           & \begin{tabular}[c]{@{}c@{}}0.99\\      (1.12e-6)\end{tabular} & \begin{tabular}[c]{@{}c@{}}0.99\\      (5.10e-7)\end{tabular} & \begin{tabular}[c]{@{}c@{}}0.99\\      (1.86e-7)\end{tabular} & \begin{tabular}[c]{@{}c@{}}0.99\\      (4.91e-8)\end{tabular} & \begin{tabular}[c]{@{}c@{}}1\\      (5.09e-17)\end{tabular} \\\midrule
    & \multicolumn{5}{c}{MATH}                                                                                                                                                                                                                                                                                                    & \multicolumn{5}{c}{GPQA}                                                                                                                                                                                                                                                                                                    \\\midrule
$\sigma^2_{a}$ & 1                                                             & 2                                                             & 3                                                             & 4                                                             & 5                                                           & 1                                                             & 2                                                             & 3                                                             & 4                                                             & 5                                                           \\
1   & \begin{tabular}[c]{@{}c@{}}1\\      (0.00e+0)\end{tabular}    &                                                               &                                                               &                                                               &                                                             & \begin{tabular}[c]{@{}c@{}}1\\      (2.55e-17)\end{tabular}   &                                                               &                                                               &                                                               &                                                             \\
2   & \begin{tabular}[c]{@{}c@{}}0.99\\      (9.23e-5)\end{tabular} & \begin{tabular}[c]{@{}c@{}}1\\      (0.00e+0)\end{tabular}    &                                                               &                                                               &                                                             & \begin{tabular}[c]{@{}c@{}}0.99\\      (1.90e-6)\end{tabular} & \begin{tabular}[c]{@{}c@{}}1\\      (2.55e-17)\end{tabular}   &                                                               &                                                               &                                                             \\
3   & \begin{tabular}[c]{@{}c@{}}0.99\\      (1.38e-4)\end{tabular} & \begin{tabular}[c]{@{}c@{}}0.99\\      (8.82e-5)\end{tabular} & \begin{tabular}[c]{@{}c@{}}1\\      (0.00e+0)\end{tabular}    &                                                               &                                                             & \begin{tabular}[c]{@{}c@{}}0.99\\      (2.21e-6)\end{tabular} & \begin{tabular}[c]{@{}c@{}}0.99\\      (1.12e-6)\end{tabular} & \begin{tabular}[c]{@{}c@{}}1\\      (3.60e-17)\end{tabular}   &                                                               &                                                             \\
4   & \begin{tabular}[c]{@{}c@{}}0.99\\      (1.50e-4)\end{tabular} & \begin{tabular}[c]{@{}c@{}}0.99\\      (1.16e-4)\end{tabular} & \begin{tabular}[c]{@{}c@{}}0.99\\      (3.94e-5)\end{tabular} & \begin{tabular}[c]{@{}c@{}}1\\      (0.00e+0)\end{tabular}    &                                                             & \begin{tabular}[c]{@{}c@{}}0.99\\      (2.72e-6)\end{tabular} & \begin{tabular}[c]{@{}c@{}}0.99\\      (1.78e-6)\end{tabular} & \begin{tabular}[c]{@{}c@{}}0.99\\      (1.60e-6)\end{tabular} & \begin{tabular}[c]{@{}c@{}}1\\      (5.70e-17)\end{tabular}   &                                                             \\
5   & \begin{tabular}[c]{@{}c@{}}0.99\\      (1.51e-4)\end{tabular} & \begin{tabular}[c]{@{}c@{}}0.99\\      (1.19e-4)\end{tabular} & \begin{tabular}[c]{@{}c@{}}0.99\\      (4.28e-5)\end{tabular} & \begin{tabular}[c]{@{}c@{}}0.99\\      (4.71e-6)\end{tabular} & \begin{tabular}[c]{@{}c@{}}1\\      (8.82e-17)\end{tabular} & \begin{tabular}[c]{@{}c@{}}0.99\\      (3.46e-6)\end{tabular} & \begin{tabular}[c]{@{}c@{}}0.99\\      (2.40e-6)\end{tabular} & \begin{tabular}[c]{@{}c@{}}0.99\\      (1.63e-6)\end{tabular} & \begin{tabular}[c]{@{}c@{}}0.99\\      (3.26e-7)\end{tabular} & \begin{tabular}[c]{@{}c@{}}1\\      (6.74e-17)\end{tabular} \\\midrule
    & \multicolumn{5}{c}{MuSR}                                                                                                                                                                                                                                                                                                    & \multicolumn{5}{c}{MMLU-Pro}                                                                                                                                                                                                                                                                                                \\\midrule
$\sigma^2_{a}$ & 1                                                             & 2                                                             & 3                                                             & 4                                                             & 5                                                           & 1                                                             & 2                                                             & 3                                                             & 4                                                             & 5                                                           \\
1   & \begin{tabular}[c]{@{}c@{}}1\\      (0.00e+0)\end{tabular}    &                                                               &                                                               &                                                               &                                                             & \begin{tabular}[c]{@{}c@{}}1\\      (5.55e-17)\end{tabular}   &                                                               &                                                               &                                                               &                                                             \\
2   & \begin{tabular}[c]{@{}c@{}}0.99\\      (2.28e-6)\end{tabular} & \begin{tabular}[c]{@{}c@{}}1\\      (0.00e+0)\end{tabular}    &                                                               &                                                               &                                                             & \begin{tabular}[c]{@{}c@{}}0.99\\      (1.95e-6)\end{tabular} & \begin{tabular}[c]{@{}c@{}}1\\      (7.85e-17)\end{tabular}   &                                                               &                                                               &                                                             \\
3   & \begin{tabular}[c]{@{}c@{}}0.99\\      (2.11e-3)\end{tabular} & \begin{tabular}[c]{@{}c@{}}0.99\\      (2.10e-3)\end{tabular} & \begin{tabular}[c]{@{}c@{}}1\\      (0.00e+0)\end{tabular}    &                                                               &                                                             & \begin{tabular}[c]{@{}c@{}}0.99\\      (2.78e-6)\end{tabular} & \begin{tabular}[c]{@{}c@{}}0.99\\      (2.33e-7)\end{tabular} & \begin{tabular}[c]{@{}c@{}}1\\      (7.85e-17)\end{tabular}   &                                                               &                                                             \\
4   & \begin{tabular}[c]{@{}c@{}}0.99\\      (2.11e-3)\end{tabular} & \begin{tabular}[c]{@{}c@{}}0.99\\      (2.11e-3)\end{tabular} & \begin{tabular}[c]{@{}c@{}}0.99\\      (5.24e-7)\end{tabular} & \begin{tabular}[c]{@{}c@{}}1\\      (0.00e+0)\end{tabular}    &                                                             & \begin{tabular}[c]{@{}c@{}}0.99\\      (3.05e-6)\end{tabular} & \begin{tabular}[c]{@{}c@{}}0.99\\      (3.77e-7)\end{tabular} & \begin{tabular}[c]{@{}c@{}}0.99\\      (5.36e-8)\end{tabular} & \begin{tabular}[c]{@{}c@{}}1\\      (5.55e-17)\end{tabular}   &                                                             \\
5   & \begin{tabular}[c]{@{}c@{}}0.99\\      (2.11e-3)\end{tabular} & \begin{tabular}[c]{@{}c@{}}0.99\\      (2.11e-3)\end{tabular} & \begin{tabular}[c]{@{}c@{}}0.99\\      (5.96e-7)\end{tabular} & \begin{tabular}[c]{@{}c@{}}0.99\\      (1.42e-7)\end{tabular} & \begin{tabular}[c]{@{}c@{}}1\\      (0.00e+0)\end{tabular}  & \begin{tabular}[c]{@{}c@{}}0.99\\      (3.20e-6)\end{tabular} & \begin{tabular}[c]{@{}c@{}}0.99\\      (4.40e-7)\end{tabular} & \begin{tabular}[c]{@{}c@{}}0.99\\      (1.21e-7)\end{tabular} & \begin{tabular}[c]{@{}c@{}}0.99\\      (2.04e-8)\end{tabular} & \begin{tabular}[c]{@{}c@{}}1\\      (0.00e+0)\end{tabular}  \\ \bottomrule
\end{tabular}
\end{table}

\subsection{Real-world benchmark details}\label{supp:benchmarkdetail}
\textbf{MATH-500.}\quad
We include the following list of LLM candidates in Table \ref{tab:llm_list} for the MATH-500 benchmark suite.
\begin{table}[t]
    \centering
    \caption{List of LLM candidates in the MATH-500 suite}
    \label{tab:llm_list}
    \scriptsize
    \begin{minipage}[t]{0.49\textwidth}
        \begin{itemize}
            \item 01-ai/Yi-34B
            \item baidu/ERNIE-4.5-21B-A3B-PT
            \item baidu/ERNIE-4.5-21B-A3B-Thinking
            \item deepseek-ai/DeepSeek-R1-0528-Qwen3-8B
            \item deepseek-ai/DeepSeek-R1-Distill-Llama-8B
            \item deepseek-ai/DeepSeek-R1-Distill-Qwen-1.5B
            \item deepseek-ai/DeepSeek-R1-Distill-Qwen-14B
            \item deepseek-ai/DeepSeek-R1-Distill-Qwen-32B
            \item deepseek-ai/DeepSeek-R1-Distill-Qwen-7B
            \item dphn/dolphin-2.9.1-yi-1.5-34b
            \item dphn/Dolphin-Mistral-24B-Venice-Edition
            \item google/gemma-2-27b-it
            \item google/gemma-2b-it
            \item google/gemma-3-1b-pt
            \item HuggingFaceTB/SmolLM3-3B
            \item huihui-ai/Huihui-gpt-oss-20b-BF16-abliterated
            \item huihui-ai/Huihui-Qwen3-8B-abliterated-v2
            \item ibm-granite/granite-3.3-2b-instruct
            \item internlm/internlm2-chat-20b
            \item LGAI-EXAONE/EXAONE-4.0.1-32B
            \item LLM360/K2-Think
            \item meta-llama/Llama-2-7b-chat-hf
            \item meta-llama/Llama-2-7b-hf
            \item meta-llama/Llama-3.1-8B-Instruct
            \item meta-llama/Llama-3.2-1B-Instruct
            \item meta-llama/Llama-3.2-3B-Instruct
            \item meta-llama/Meta-Llama-3-8B-Instruct
            \item microsoft/Phi-3.5-mini-instruct
            \item microsoft/Phi-4-mini-instruct
            \item microsoft/Phi-4-mini-reasoning
            \item microsoft/Phi-4-mini-reasoning-plus
            \item mistralai/Mistral-7B-Instruct-v0.1
            \item mistralai/Mistral-7B-Instruct-v0.2
            \item mistralai/Mistral-7B-Instruct-v0.3
            \item mistralai/Mistral-Small-Instruct-2409
        \end{itemize}
    \end{minipage}
    \hfill
    \begin{minipage}[t]{0.49\textwidth}
        \begin{itemize}
            \item moonshotai/Moonlight-16B-A3B
            \item moonshotai/Moonlight-16B-A3B-Instruct
            \item nvidia/AceReason-Nemotron-1.1-7B
            \item nvidia/AceReason-Nemotron-14B
            \item nvidia/Llama-3.1-Nemotron-8B-UltraLong-4M-Instruct
            \item nvidia/Nemotron-Research-Reasoning-Qwen-1.5B
            \item nvidia/NVIDIA-Nemotron-Nano-12B-v2
            \item nvidia/NVIDIA-Nemotron-Nano-9B-v2
            \item nvidia/OpenReasoning-Nemotron-1.5B
            \item nvidia/OpenReasoning-Nemotron-7B
            \item openai/gpt-oss-20b
            \item openbmb/MiniCPM4.1-8B
            \item Qwen/Qwen1.5-32B
            \item Qwen/Qwen2.5-0.5B-Instruct
            \item Qwen/Qwen2.5-1.5B-Instruct
            \item Qwen/Qwen2.5-14B-Instruct
            \item Qwen/Qwen2.5-32B-Instruct
            \item Qwen/Qwen2.5-3B-Instruct
            \item Qwen/Qwen2.5-7B-Instruct
            \item Qwen/Qwen2-7B-Instruct
            \item Qwen/Qwen3-0.6B
            \item Qwen/Qwen3-1.7B
            \item Qwen/Qwen3-4B-Thinking-2507
            \item Qwen/Qwen3-8B
            \item Qwen/QwQ-32B
            \item Qwen/Qwen3-14B
            \item Qwen/Qwen3-30B-A3B
            \item Qwen/Qwen3-30B-A3B-Instruct-2507
            \item Qwen/Qwen3-30B-A3B-Thinking-2507
            \item Qwen/Qwen3-32B
            \item Qwen/Qwen3-4B
            \item Qwen/Qwen3-4B-Instruct-2507
            \item swiss-ai/Apertus-8B-Instruct-2509
            \item zai-org/GLM-4-32B-0414
            \item zai-org/GLM-4-9B-0414
        \end{itemize}
    \end{minipage}
\end{table}
Inference parameters as in \cite{lart} are set as follows: temperature is 0.5, top-p is 0.95, and repetition penalty is 1.05, with a maximum output length of 10,240 tokens. The adopted prompt templates are shown in Table \ref{tab:prompts}, where \{problem\} represents the specific question for each item.
\begin{table}[t]
\centering
\caption{Prompt templates used in MATH-500 evaluation.}
\label{tab:prompts}
\scriptsize
\renewcommand{\arraystretch}{1.5}
\begin{tabular}{|p{0.45\textwidth}|p{0.45\textwidth}|}
\hline
\textbf{Zero-shot} & \textbf{One-shot} \\ \hline
Solve the following math problem. Be clear and concise. Problem: ``\{problem\}'' Provide a \textbf{step-by-step solution}. Start each step with a number followed by a period (e.g., `1.', `2.', etc.). Use basic LaTeX for mathematical expressions, such as for fractions, exponents, and variables. Avoid complex formatting. At the very end of your entire response, and only at the very end, state the final answer. This final answer must be enclosed in a single LaTeX box, like so: \fbox{Answer}.
&
Solve the following math problem. Please think \textbf{step-by-step} to obtain the solution. Use basic LaTeX for mathematical expressions, such as for fractions, exponents, and variables. Avoid complex formatting. At the very end of your entire response, and only at the very end, state the final answer. This final answer must be enclosed in a single LaTeX box, like so: \fbox{Answer}.

Here is an example of how to format your response and think about solving the problem:

Example Problem: What is the sum of the two values of $x$ for which $(x + 3)^2 = 121$?

Example Solution: Expanding the left side, we have $x^2 + 6x + 9 = 121 \Rightarrow x^2 + 6x - 112 = 0$. For a quadratic with the equation $ax^2 + bx + c = 0$, the sum of the roots is $-b/a$. Applying this formula to the problem, we have that the sum of the two roots is $-6/1 = \fbox{$-6$}$.

Solution: \fbox{$-6$}

New Problem: \{problem\}. \\ \hline
\end{tabular}
\end{table}

\textbf{Hugging Face Open LLM Leaderboard.}\quad
We refer to \cite{wu2026} when formulating response matrix from Hugging Face's Open LLM Leaderboard. Specifically, we utilized the raw historical data $M_1$ from \url{https://github.com/skbwu/efficiently-evaluating-llms}, comprising item-level accuracy metadata for six major benchmarks: IFEval, BBH, MATH, GPQA, MuSR, and MMLU-Pro. The dataset was retrieved through the \texttt{huggingface\_hub} API by iterating through a cohort of 2211 recent evaluation repositories, spanning the period from June to December 2024, under the Leaderboard v2 schema.

\textbf{Ranking flow details.}\quad
Full ranking information based on average accuracy and recovered ability across the seven benchmarks is provided in the Supplementary Material (\texttt{.csv}).

\end{document}